\documentclass[5p,times]{elsarticle}
\usepackage{fixmath}
\usepackage{amsmath, amsthm, amssymb, mathrsfs}
\usepackage{mathdots}
\usepackage{graphicx}
\graphicspath{{images_pdf/}}
\usepackage[all]{xy}
\usepackage{xcolor}
\usepackage{supertabular}
\usepackage{physics}

\usepackage{caption}
\usepackage{subcaption}

\usepackage{algorithm2e}

\usepackage{nicematrix}

\usepackage{hyperref}

\newtheorem{proposition}{Proposition}
\newtheorem{corollary}{Corollary}

\journal{Mechatronics}

\newcommand{\R}{\mathbb R}

\newdefinition{rmk}{Remark} 
\newproof{pf}{Proof} 
\newproof{pot}{Proof of Theorem \ref{thm2}}

\begin{document}

\begin{frontmatter}
\title{An Iterative Algorithm to Symbolically Derive Generalized n-Trailer Vehicle Kinematics}

\author[label1,label3]{Yuvraj Singh\corref{cor1}}\ead{singh.1250@osu.edu} 
\author[label2,label3]{Adithya Jayakumar} 
\author[label1,label3]{Giorgio Rizzoni} 

\cortext[cor1]{Corresponding author}

\affiliation[label1]{organization={Department of Mechanical and Aerospace Engineering, The Ohio State University},
            addressline={201 W. 19th Avenue}, 
            city={Columbus},
            postcode={43210}, 
            state={OH},
            country={USA}}
\affiliation[label2]{organization={Department of Engineering Education, The Ohio State University},
            addressline={174 W. 18th Ave}, 
            city={Columbus},
            postcode={43210}, 
            state={OH},
            country={USA}}
\affiliation[label3]{organization={The Ohio State University Center for Automotive Research},
            addressline={930 Kinnear Road}, 
            city={Columbus},
            postcode={43212}, 
            state={OH},
            country={USA}}

\begin{abstract}
Articulated multi-axle vehicles are interesting from a control-theoretic perspective due to their peculiar kinematic offtracking characteristics, instability modes, and singularities. 
Holonomic and nonholonomic constraints affecting the kinematic behavior is investigated in order to develop control-oriented kinematic models representative of these peculiarities.
Then, the structure of these constraints is exploited to develop an iterative algorithm to symbolically derive yaw-plane kinematic models of generalized $n$-trailer articulated vehicles with an arbitrary number of multi-axle vehicle units. 
A formal proof is provided for the maximum number of kinematic controls admissible to a large-scale generalized articulated vehicle system, which leads to a generalized Ackermann steering law for $n$-trailer systems. 
Moreover, kinematic data collected from a test vehicle is used to validate the kinematic models and, to understand the rearward yaw rate amplification behavior of the vehicle pulling multiple simulated trailers.

\end{abstract}




\begin{keyword}
Mobile robots and vehicles \sep Articulated vehicles \sep Nonholonomic robots \sep Kinematics \sep Lie Group methods
\end{keyword}
            
\end{frontmatter}

\section{Introduction}
Multi-body articulated vehicle systems or tractor-trailer vehicles consist of one or more vehicle units (trailers, dollies) pulled by a tractor (or cab) unit. Articulated vehicles provide a smaller turning radius and improved maneuverability as compared to a single-unit vehicle of equivalent wheelbase. However, the improvements in maneuverability and logistical efficiency come at the cost of introducing complicated kinematic and dynamic behavior (refer Figure \ref{fig:tractor_trailer_instabilities}) such as offtracking, jackknifing, rearward amplification, and higher rollover propensity 
\cite{trigell_truck_trailer_dynamics, kurtz_anderson_survey, vlk_lateral_dynamics, jindra_1966, mikulcik_1971}. 
High-fidelity models for articulated vehicle dynamics are widely available in the literature \cite{truck9dof,genta} as well as in commercial vehicle simulation packages 
\cite{trucksim,truckmaker,vdbs}
to analyze the dynamic behavior of such vehicles. Such models provide test bench for design and development of vehicle systems like suspension, powertrain, braking systems, steering, vehicle dynamics controllers, etc. 
However, being too complicated, these models are not suitable for control synthesis of motion planning and tracking controllers for an automated driving system. 
Instead, simplified kinematic reduction models have been used for developing motion planning algorithms and to establish controllability properties of these dynamic systems \cite{kinematic_reduction_controllability_1}.

Two types of tractor-trailer systems are commonly discussed in the literature related to nonholonomic path planning. Standard $n$-trailer systems consist of $n$ trailers being pulled by a tractor unit where all trailers are hitched on the rear axle of the vehicle unit in front of them. General $n$-trailers also consist of $n$ trailers pulled by a tractor unit, but the trailers can have off-axle hitching.
Controllability for standard $n$-trailer has been proven by Laumond \cite{laumond_controllability}. Path planning approaches for the standard $n$-trailer system have been extensively explored in the literature using chained form transformations 
\cite{chained_form_1,chained_form_sinusoids_1,chained_form_sinusoids_2} 
and by exploiting the differentially flat structure of the tractor-trailer planar kinematic models \cite{flatness_motion_planning_1}. However, most tractor-trailer systems in use are not standard $n$-trailer systems since they have off-axle hitching. The general $n$-trailer systems and their properties are described in \cite{general_n_trailer_properties}, and planning algorithms for them have been explored in \cite{flatness_motion_planning_2}. Due to off-axle hitching, the general $n$-trailer system (for $n>1$) loses the nice properties of differential flatness and the chained form that standard $n$-trailer systems exhibit \cite{flatness_motion_planning_2}. While the general $n$-trailer is controllable \cite{general_n_trailer_properties}, it has singular configurations when consecutive trailers become orthogonally oriented, which causes the Lie brackets of vector fields associated with vehicle kinematics to not form a full-rank distribution \cite{general_n_trailer_properties}. 

\begin{figure}
    \centering
    \includegraphics[width=\linewidth]{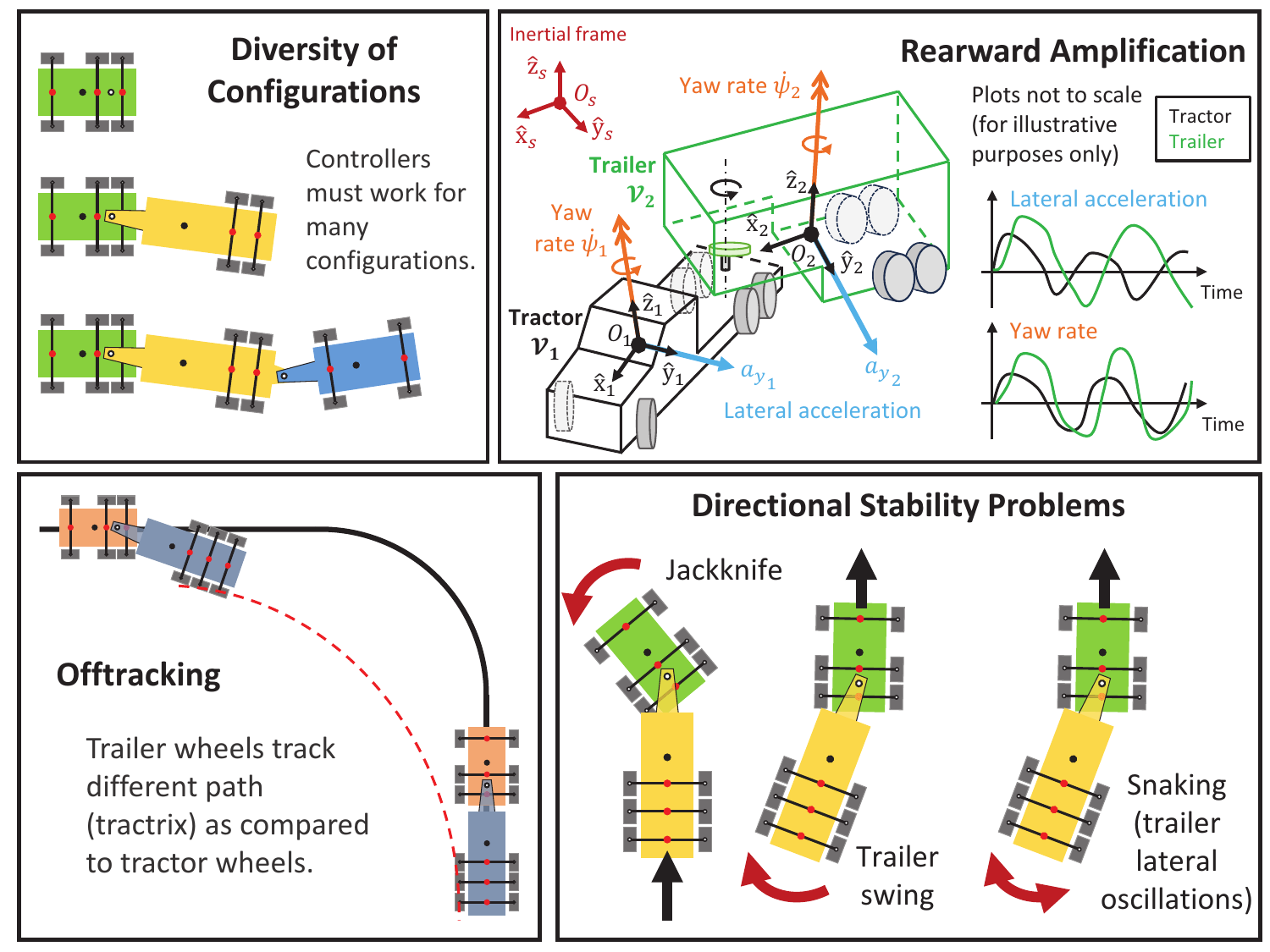}
    \caption{Challenges involved in control of articulated vehicle systems}
    \label{fig:tractor_trailer_instabilities}
\end{figure}

In order to understand the control-theoretic properties of articulated vehicle systems, nonlinear control-oriented kinematic models for generalized $n$-trailer articulated vehicle systems are developed in this paper. 
Throughout this paper, the term ``\textit{generalized articulated vehicle system}" means a general $n$-trailer vehicle system with an arbitrary number of vehicle units with revolute joint hitches connecting adjacent vehicle units, where each vehicle unit has an arbitrary number of axles. 


While dynamic models for tractor-trailer vehicles have been widely explored 
\cite{deBruin_dynamic_model,genta,truck9dof}, and have been used for developing controllers \cite{kim_control_all_steered,yawrate_rwa_mpc}
, systematic approaches for kinematic modeling have been less explored.
Kinematic reductions of dynamic models are of immense importance for deriving control-theoretic notions such as controllability \cite{kinematic_reduction_controllability_1}, an important property for trajectory and motion planning. Kinematic models have been extensively used for trajectory planning \cite{oliviera} as well as path tracking \cite{astolfi_path_tracking}.

Unlike currently available iterative kinematic modeling methods found in the literature \cite{michalek,Orosco_modeling_feedback_linearization_n_trailer,tilbury_multisteering,michalek}, no assumptions have been made about the geometry of wheels and hitches in this paper.
%
%
As a result, the modeling procedure described in this paper is capable of representing the kinematics of a large diversity of vehicle configurations.
From a practical perspective, a diversity of configurations is a feature of articulated commercial vehicle systems commonly used for goods transportation. 
For instance, a tractor unit may tow more than one trailer in a highway driving scenario. In another scenario, the same tractor unit may be used for reversing a single trailer in a docking station. Hence, vehicle models used for developing automated driving algorithms for such vehicles must be robust to changes in the configuration of the system. 

\subsection{Contributions of this paper}
This paper has the following contributions:
\begin{enumerate}
    \item A modular iterative approach to derive Jacobian velocity kinematics for general $n$-trailer articulated vehicles applicable to vehicles with on-axle hitching, off-axle hitching (semitrailer, full trailer), laterally offset hitches, multiple axles, and exotic wheel geometries. Modular approaches to kinematic modeling of articulated buses under various assumptions related to vehicle and hitch geometry have been developed by Michalek et al. \cite{michalek}. However, in this paper, we do not make assumption on the vehicle's configuration. 
    \item 
    The structure of the nonholonomic constraints is utilized to develop a multi-axle $n$-trailer generalization of the extended Ackermann steering rule \cite{genta}. While, an extended Ackermann steering rule for a vehicle with trailer is found in the literature \cite{genta}, a general rule for multi-axle $n$-trailer vehicles has not been developed. 
    \item Despite being one of the most critical safety concerns for articulated vehicle control, Rearward amplification behavior has not been extensively explored for $n$-trailer systems in the literature. Moreover, rearward amplification is an important concern for vehicle dynamic/kinetic modeling \cite{ahmadian_rwa}. However this notion has not been explored for kinematic models which are predominantly used for control design for motion planning. Hence, an investigation of rearward yaw rate amplification of an underactuated vehicle towing multiple simulated trailers using real-world driver input data collected on an experimental single-unit passenger vehicle.
    
    
\end{enumerate}

\section{Kinematic Constraints on an Articulated Vehicle}
An articulated vehicle system moving in a plane can be modeled as a collection of rigid bodies, with consecutive chassis units connected by revolute joints, where each chassis unit has several wheels (assumed to be rigid) mounted on it.
\subsection{Preliminaries: Rigid Body Motion}

The rotation transformation using the Special Orthogonal $\mathrm{SO}(n)$ Lie group and homogeneous transformation using the Special Euclidean $\mathrm{SE}(n)$ Lie group are the fundamental transformations that are used to define the kinematic relationships in a robotic system composed of rigid bodies 
\footnote{
For more details on these Lie groups and their applications in robot kinematics development, interested readers shall refer to \cite{lynch_park_robotics}.
}.
The $\mathrm{SO}(n)$ group represents rotation transformations and is formally defined as the set of all $n \times n$ real matrices with determinant one, as defined in equation \eqref{eq:SOn_definition}. The $\mathrm{SE}(n)$ group represents rigid body motions, that is, both translation and rotation transformations, and is defined formally in equation \eqref{eq:SEn_definition}.
\begin{equation}
    \mathrm{SO}(n) = \left\{ R \in \R^{n \times n} \text{ s.t. } R^\top R = I_n \, , \ \mathrm{det}R = 1 \right\}
    \label{eq:SOn_definition}
\end{equation}
\begin{equation}
    \mathrm{SE}(n) = \left\{ T := \left( \begin{array}{c|c}
                            R           & \textbf{p} \\ \hline
                            \textbf{0}  & 1
                        \end{array} \right)
    \text{ s.t. } R \in \mathrm{SO}(n) \, , \ \textbf{p} \in \R^n \right\}
    \label{eq:SEn_definition}
\end{equation}

\subsubsection[Planar Rotations: \mathrm{SO}(2) Group]{Planar Rotations: $\mathrm{SO}(2)$ Group}
A vector $\begin{bmatrix} a & b \end{bmatrix}^\top$ is rotated by an angle $\psi$ about the normal to the plane ($\hat{\mathrm{z}}$-axis). The transformed vector (after rotation) $\begin{bmatrix} a_r & b_r \end{bmatrix}^\top$ is given by equation \eqref{eq:SO2_transform}
\begin{equation}
    \begin{pmatrix} a_r \\ b_r \end{pmatrix} = 
    R(\psi) \begin{pmatrix} a \\ b \end{pmatrix}
    , \ R(\psi) := \begin{pmatrix}
        \cos{\psi} & -\sin{\psi} \\
        \sin{\psi} &  \cos{\psi} 
    \end{pmatrix} \in \mathrm{SO}(2)
        \label{eq:SO2_transform}
\end{equation}

The following results related to the $R(\psi) \in \mathrm{SO}(2)$ rotation matrices can be derived using equation \eqref{eq:SO2_transform}: 
%
inverse (equation \eqref{eq:SO2_inverse}), time derivative (equation \eqref{eq:SO2_timederivative}), and a sequence of planar rotations by angles $\psi_1$ and $\psi_2$ (equation \eqref{eq:SO2_sequence}).
\begin{equation}
    R(\psi)^{-1} = R(\psi)^\top = R(-\psi)
    \label{eq:SO2_inverse}
\end{equation}
\begin{equation}
    \dv{t} R(\psi(t)) = R\left(\psi(t) + \frac{\pi}{2}\right) \dot{\psi}(t)
    \label{eq:SO2_timederivative}
\end{equation}
\begin{equation}
    R(\psi_2)R(\psi_1) = R(\psi_1 + \psi_2)
    \label{eq:SO2_sequence}
\end{equation}

\subsubsection[Planar Rigid Motions: \mathrm{SE}(2) Group]{Planar Rigid Motions: $\mathrm{SE}(2)$ Group} 
Consider a point $A$ located on a rigid body $\mathcal{V}_i$ having coordinates $(a_1,a_2)$ with respect to body frame $\{i\}$.
The frame $\{i\}$ is located at coordinates $\mathrm{p}_i := (x_i,y_i) \in \mathbb{R}^2$ and is aligned at an angle $\psi_i$ with respect to inertial frame $\{s\}$. 
The inertial frame coordinates $(x_a,y_a)$ of the point $A$ are given by equation \eqref{eq:SE2_transform}.
\begin{align}
    \begin{pmatrix} x_a & y_a & 1 \end{pmatrix}^\top
     &= T_{s,i}\begin{pmatrix} a_1 & a_2 & 1 \end{pmatrix}^\top
     \label{eq:SE2_transform}
     \\
    \text{where, }
    T_{s,i} =
    T(\psi_i,\textbf{p}_i) &:= 
        \left(
        \begin{array}{c|c}
            R(\psi_i) & \textbf{p}_i \\ \hline
            \textbf{0} & 1
        \end{array}
        \right)
        \in \mathrm{SE}(2) 
    \nonumber
\end{align}

\subsection{Specifying an Articulated Vehicle}
A generalized $n$-trailer articulated vehicle system can be described as an ordered sequence $(\mathcal{V}_1, \mathcal{V}_2 , \dots \mathcal{V}_n)$ of several multibody vehicle units $\mathcal{V}_i,\, i \in \{1,\dots,n\}$ as shown in Figure \ref{fig:holonomic}. 
Each vehicle unit $\mathcal{V}_i := (C_i,W_i)$ is composed of a chassis $C_i$ and a set of wheels $W_i:= \{\mathcal{W}_{i,1},\dots,\mathcal{W}_{i,K_i}\}$. 
All vehicle unit chassis $C_i$, and wheels $\mathcal{W}_{i,k}$ are modeled as rigid bodies moving in a plane, and hence, can be endowed with a Special Euclidean $\mathrm{SE}(2)$ matrix Lie group structure in order to describe their configuration. 
A vehicle unit $\mathcal{V}_i$'s chassis $C_i$ is endowed with the matrix Lie group structure $T_{s,i}\in \mathrm{SE}(2)$ as described in equation \eqref{eq:chassis_group}, where $\textbf{p}_i:=(x_i,y_i)\in \mathbb{R}^2$ and $\psi_i \in S^1 := [0,2\pi)$ are respectively, the location and orientation of the chassis $C_i$'s body frame of reference $\{i\}$ with respect to some arbitrary inertial frame of reference $\{s\}$.
\begin{align}
    C_i 
    \leftrightarrow 
    T_{s,i} &= \left(\begin{array}{c|c}
            R(\psi_i) & \textbf{p}_i \\ \hline
            \textbf{0} & 1
          \end{array}\right) \in \mathrm{SE}(2) \label{eq:chassis_group} \\
      \text{where } &
      R(\psi_i)
      \in \mathrm{SO}(2), \, \textbf{p}_i \in \mathbb{R}^2 , \, \psi_i \in S^1
      \nonumber
\end{align}
A wheel $\mathcal{W}_{i,k}$ located on the vehicle unit $\mathcal{V}_i$ is endowed with the Lie group structure $T_{s,w_{i,k}} \in \mathrm{SE}(2)$ described in equation \eqref{eq:wheel_group}:
\begin{equation*}
    W_{i,k} 
    \leftrightarrow 
    T_{s,w_{i,k}} = T_{s,i} T_{i,w_{i,k}} 
                  = T\left(\psi_i,\textbf{p}_i\right) T\left(\theta_{i,k},\textbf{w}_{i,k}\vert_{\{i\}}\right) \in \mathrm{SE}(2) 
\end{equation*}
\begin{equation}
    \implies
    T_{s,w_{i,k}} = \left(\begin{array}{c|c}
        R(\psi_i+\theta_{i,k}) & \textbf{p}_i + R(\psi_i) \textbf{w}_{i,k}\vert_{\{i\}} \\ \hline
        \textbf{0} & 1
      \end{array}\right) 
      \label{eq:wheel_group}
\end{equation}
In equation \eqref{eq:wheel_group}, the wheel's location on the vehicle unit $\mathcal{V}_i$ with respect to the body frame of reference $\{i\}$ is denoted by $\textbf{w}_{i,k}\vert_{\{i\}}:=(a_{i,k},b_{i,k})\in \mathbb{R}^2$, and is constant since the wheel is mounted on the chassis and only rotates about the wheel's vertical $\hat{z}$-axis. In equation \eqref{eq:wheel_group}, $\theta_{i,k} \in S^1 := [0,2\pi)$ is orientation (steering angle) of the wheel with respect to the chassis $C_i$.
As observed in equation \eqref{eq:wheel_group}, the wheel $\mathcal{W}_{i,k}$ is oriented at an angle of $(\psi_i + \theta_{i,k})$ with respect to the inertial frame of reference $\{s\}$.
Let $\textbf{w}_{i,k}\vert_{\{s\}}:=(x_{w_{i,k}},y_{w_{i,k}})\in \mathbb{R}^2$ be the location of the wheel $\mathcal{W}_{i,k}$ with respect to the inertial frame of reference $\{s\}$,  
From equation \eqref{eq:wheel_group}, the inertial location of the wheel $\mathcal{W}_{i,k}$ is determined as follows:
\begin{equation}
    \textbf{w}_{i,k}\vert_{\{s\}}
    =
    \textbf{p}_i + R(\psi_i) \textbf{w}_{i,k}\vert_{\{i\}}
    \label{eq:wheel_inertial_location}
\end{equation}


\begin{figure}
    \centering
    \includegraphics[width=\linewidth]{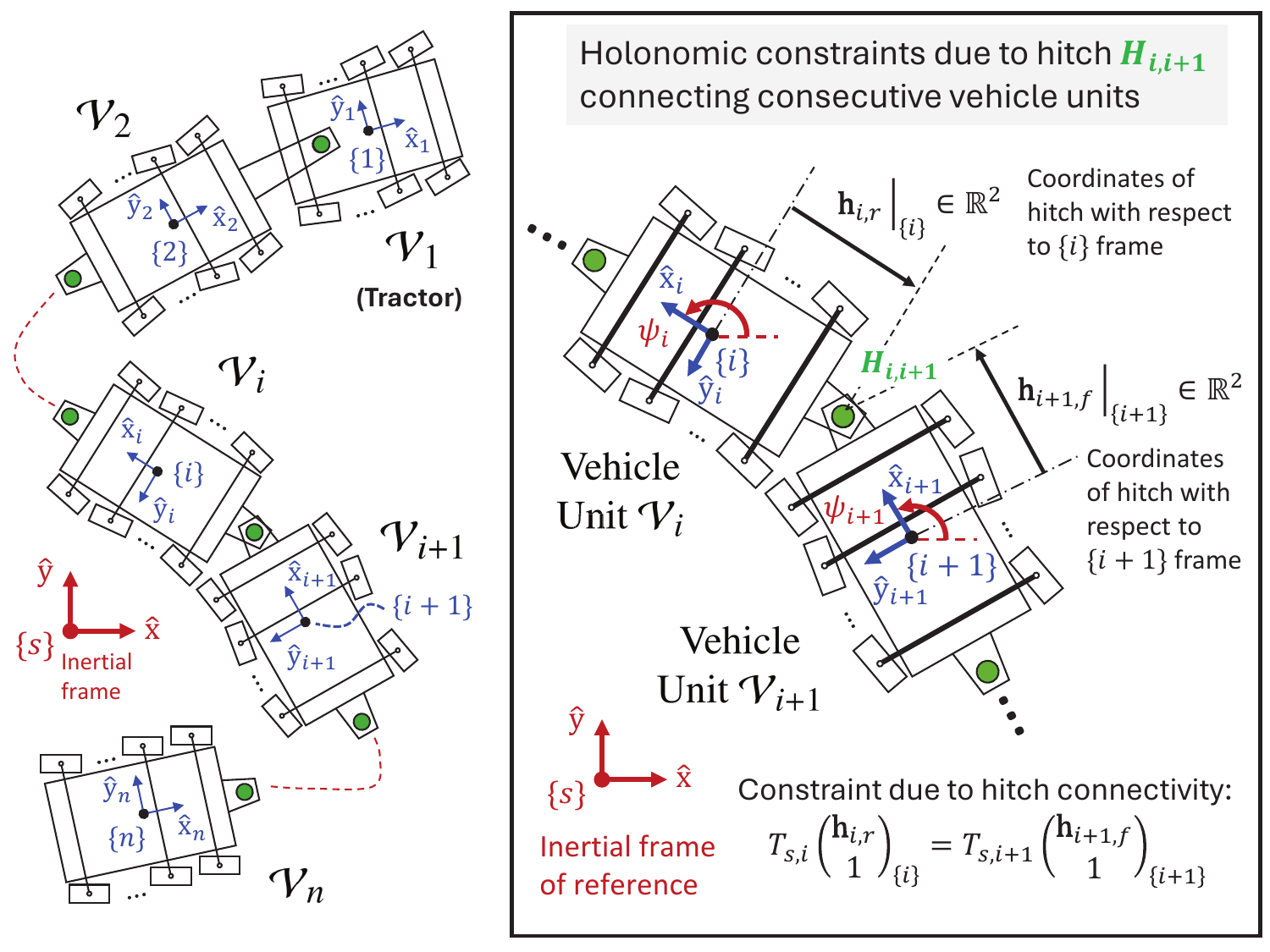}
    \caption{Articulated vehicle system and holonomic constraints: \\ 
    (Left): The articulated vehicle system is schematically shown, and its various frames of reference are defined here. The inertial frame $\{s\}$ is placed arbitrarily, and body frames $\{i\}$ are shown for each vehicle unit. \\
    (Right): The illustration shows the holonomic constraint between two adjacent vehicle units due to hitch connectivity.}
    \label{fig:holonomic}
\end{figure}

\subsection{Holonomic constraints}
The vehicle system's configuration is constrained due to the hitch revolute joints connecting adjacent vehicle units. Consider the hitch joint denoted by $H_{i,i+1}$ connecting vehicle units $\mathcal{V}_i$ and $\mathcal{V}_{i+1}$.
The hitch connectivity constraint asserts that the coordinates of the hitch $H_{i,i+1}$ with respect to the inertial frame $\{s\}$, obtained by transforming hitch coordinates with respect to either of the adjacent vehicle chassis frames $\{i\}$ and $\{i+1\}$ using appropriate homogeneous transformations shall be equal as shown in Figure \ref{fig:holonomic}. 
Mathematically, this is expressed as follows:
\begin{equation*}
    \begin{pmatrix} H_{i,i+1} \\ 1 \end{pmatrix}_{\{s\}}
    =
    T_{s,i}
    \left(\begin{array}{c}
       \textbf{h}_{i,r}  \\  1
    \end{array}\right)_{\{i\}}
    =
    T_{s,i+1}
    \left(\begin{array}{c}
       \textbf{h}_{i+1,f}  \\  1
    \end{array}\right)_{\{i+1\}}
\end{equation*}
where,
$\textbf{h}_{i,r} := H_{i,i+1}\vert_{\{i\}}$ are the coordinates of the rear hitch located on vehicle unit $\mathcal{V}_i$ with respect to chassis frame $\{i\}$. 
Similarly, $\textbf{h}_{i+1,f} := H_{i,i+1}\vert_{\{i+1\}}$ are the coordinates of the front hitch located on vehicle unit $\mathcal{V}_{i+1}$ with respect to chassis frame $\{i+1\}$. 
Hence, the above constraint equation leads to the following holonomic constraint $\forall i\in\{1,\dots,n-1\}$:
\begin{equation}
    \textbf{p}_{i+1} = \textbf{p}_i + R(\psi_i) \textbf{h}_{i,r} - R(\psi_{i+1}) \textbf{h}_{i+1,f} 
    \label{eq:holonomic_constraint_iterative}
\end{equation}
Hence, it is observed in equation \eqref{eq:holonomic_constraint_iterative} that, given the inertial coordinates and orientation of unit $\mathcal{V}_i$, the inertial coordinates of its trailing unit $\mathcal{V}_{i+1}$ are dependent only on its orientation (yaw) angle $\psi_{i+1}$.
Performing the iterative computation described in equation \eqref{eq:holonomic_constraint_iterative} for $i = 1$ to $n-1$, the expression shown in equation \eqref{eq:holonomic_constraint} is obtained, which describes the holonomic constraint for each trailer unit in terms of the tractor location and orientation angle of all vehicle units leading the trailer.
Hence, for all trailer units, that is, $\forall i \in \{2,\dots,n\}$:
\begin{equation}
    \textbf{p}_i = \textbf{p}_1 
                    + \left( \sum_{m=1}^{i-1} R(\psi_m) \left( \textbf{h}_{m,r} - \textbf{h}_{m,f} \right) \right)
                    - R(\psi_{i}) \textbf{h}_{i,f}
    \label{eq:holonomic_constraint}
\end{equation}
Since, the notion of a front hitch on a tractor unit is trivial, in equation \eqref{eq:holonomic_constraint}, $h_{1,f} := 0$.

Hence, the Lie group structure of the chassis units $C_i$ can be updated by substituting equation \eqref{eq:holonomic_constraint} into equation \eqref{eq:chassis_group}, thus obtaining equation \eqref{eq:trailer_group}:
\begin{align}
    \text{Tractor: }
    C_1 
    &\leftrightarrow 
    T_1 = \left(\begin{array}{c|c}
            R(\psi_1) & \textbf{p}_1 \\ \hline
            \textbf{0} & 1
          \end{array}\right)
      \hspace{0.05in} , \,
      \textbf{p}_1 := \begin{pmatrix}
                          x_1 \\ y_1
                      \end{pmatrix}
     \label{eq:tractor_group} 
    \\
    \text{Trailers: }
    C_i 
    &\leftrightarrow 
    T_{s,i}= \left(\begin{array}{c|c}
            R(\psi_i) & \textbf{p}_i \\ \hline
            \textbf{0} & 1
          \end{array}\right) 
    \hspace{0.05in} , \ \forall i\geq2
     \label{eq:trailer_group} 
    \\
    \text{where, } & \textbf{p}_i\left(\textbf{p}_1,\psi_1, \dots, \psi_i \right) \text{: refer equation \eqref{eq:holonomic_constraint}}
    \nonumber
\end{align}

\subsection{Configuration space}
It is observed in equation \eqref{eq:trailer_group} that due to the hitch joint (holonomic constraint: equation \eqref{eq:holonomic_constraint}) between two adjacent vehicles, specifying the location $\textbf{p}_i = (x_i,y_i), \, \forall i>1$ of trailer units ($\mathcal{V}_i , \ i>1$) is not necessary to define the minimal coordinates needed to  define the configuration of the entire vehicle system. This is because, $\textbf{p}_i$ is determined from tractor position $\textbf{p}_1 \in \mathbb{R}^2$ and yaw angles $\{\psi_1,\dots,\psi_i\} , \, \forall i \leq n$,  where $\psi_j \in S^1 := [0,2\pi) , \, \forall j$ using equation\eqref{eq:holonomic_constraint}. 
%
Similarly, from equation \eqref{eq:wheel_group}, it can be concluded that the configuration of a wheel $\mathcal{W}_{i,k}$ can be uniquely defined by using $\textbf{p}_1 \in \mathbb{R}^2$, $\{\psi_1,\dots,\psi_i\} , \, \forall i \leq n$, and the steering angle $\theta_{i,k} \in S^1$. 

Hence, the generalized coordinates $\textbf{q}$ required to specify the configuration of the vehicle system (chassis + wheels) are defined in equation \eqref{eq:genCoordinates} as follows:

\begin{align}
    \textbf{q} = [ &\underbrace{x_1,y_1}_{\text{location of } \mathcal{V}_1}, \ \  
            \underbrace{\psi_1,\psi_2,\dots,\psi_N}_{\text{yaw angles of } \mathcal{V}_i},  
            \underbrace{\theta_{1,1},\dots,\theta_{1,K_1}}_{\text{steering angles of } \mathcal{V}_1} \ \ , \nonumber \\ 
           &\underbrace{\theta_{2,1},\dots,\theta_{2,K_2}}_{\text{steering angles of } \mathcal{V}_2}, \
           \dots \ , \ \underbrace{\theta_{n,1},\dots,\theta_{n,K_n}}_{\text{steering angles of } \mathcal{V}_n} ]^\top \in \mathcal{C} \label{eq:genCoordinates} \\
           & \text{where, } \textbf{p}_1 = (x_1,y_1) \in \mathbb{R}^2, \ \ \ \psi_i, \theta_{i,k} \in S^1 := [0,2\pi) \nonumber
\end{align}
The configuration space $\mathcal{C}$ is determined in equation \eqref{eq:Cspace} from the generalized coordinates $\textbf{q}$, thus showing that a generalized articulated vehicle is a $\left( 2 + n + \sum_{i=1}^n K_i \right)$-dimensional system.
\begin{equation*}
    \mathcal{C} = \mathbb{R}^2 \times \underbrace{S^1 \times \dots \times S^1}_{n \text{ times}} \times
                   \underbrace{S^1 \times \dots \times S^1}_{K_1 \text{ times}} \dots \times \underbrace{S^1 \times \dots \times S^1}_{K_n \text{ times}}
\end{equation*}
\begin{equation}
   \implies \mathcal{C} = \mathbb{R}^2 \times \mathbb{T}^{ \left( n+\sum_{i=1}^n K_i \right) }  
   \implies \mathrm{dim} \, \mathcal{C} = 2 + n+\sum_{i=1}^n K_i 
   \label{eq:Cspace}
\end{equation}

\subsection{Nonholonomic Pfaffian constraints}
\begin{figure}
    \centering
    \includegraphics[width=0.7\linewidth]{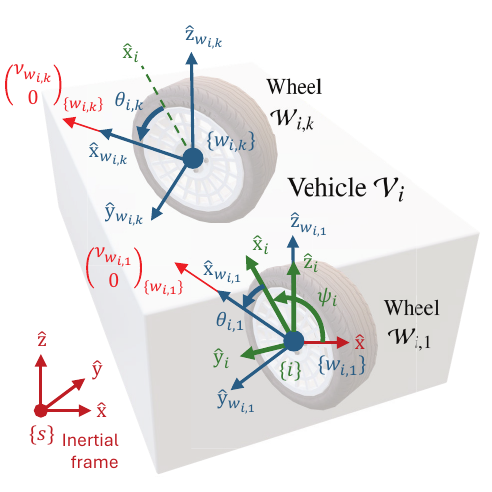}
    \caption{Illustration: Nonholonomic constraint. Notice that due to the assumption of rolling without slipping, the velocity vector for each wheel must be oriented along the longitudinal axis of the wheel frame of reference.}
    \label{fig:nonholonomic}
\end{figure}
Nonholonomic constraints 
do not restrict the configuration space of the system, but lead to inadmissibility of certain directions of movement (velocities).
An arbitrary wheel $\mathcal{W}_{i,k}$ rolling without slipping can only move in the $\hat{\mathrm{x}}$-axis of the wheel frame of reference $\{w_{i,k}\}$ (refer Figure \ref{fig:nonholonomic}). Hence, the velocity of the wheel with respect to $\{w_{i,k}\}$ is $(v_{w_{i,k}},0)\vert_{\{w_{i,k}\}}$.  
Since the frame $\{w_{i,k}\}$ is oriented at an angle of $(\psi_i + \theta_{i,k})$, the velocity vector $(\dot{x}_{w_{i,k}},\dot{y}_{w_{i,k}})_{\{s\}}$ of the wheel's road contact point with respect to inertial frame $\{s\}$ is obtained by transforming using the rotation matrix $R(\psi_i + \theta_{i,k}) \in \mathrm{SO}(2)$ as shown in equation \eqref{eq:wheel_velocity_rotation}.
\begin{equation}
    \begin{pmatrix}
        \dot{x}_{w_{i,k}} \\ \dot{y}_{w_{i,k}} 
    \end{pmatrix}
    = R(\psi_i +\theta_{i,k} ) \begin{pmatrix} v_{w_{i,k}}  \\ 0 \end{pmatrix}
    \label{eq:wheel_velocity_rotation}
\end{equation}
Using equation \eqref{eq:SO2_inverse}:
\begin{equation}
    \implies
    \begin{pmatrix} v_{w_{i,k}}  \\ 0 \end{pmatrix}
    =  R(\psi_i +\theta_{i,k} )^\top \begin{pmatrix} \dot{x}_{w_{i,k}} \\ \dot{y}_{w_{i,k}} \end{pmatrix} 
    =  R(-\psi_i -\theta_{i,k} ) \begin{pmatrix} \dot{x}_{w_{i,k}} \\ \dot{y}_{w_{i,k}} \end{pmatrix} \label{nonholonomic_wheel}
\end{equation}
Using equations \eqref{eq:wheel_inertial_location} and \eqref{eq:holonomic_constraint}
\begin{align*}
    \begin{pmatrix}
        \dot{x}_{w_{i,k}} \\ \dot{y}_{w_{i,k}} 
    \end{pmatrix}
    &= \dv{t} \textbf{w}_{i,k} \vert_{\{s\}} 
        = \dv{t}\textbf{p}_1 
        + \left( \sum_{m=1}^{i-1}  \dv{t}(R(\psi_m ))  \left( \textbf{h}_{m,r} - \textbf{h}_{m,f} \right) \right) 
        \\
        & \ \ \ \ \ \ + \dv{t}(R(\psi_{i})) \left(  \textbf{w}_{i,k}\vert_{\{i\}} - \textbf{h}_{i,f}  \right) 
\end{align*}
Using equation \eqref{eq:SO2_timederivative}
\begin{align*}
    \begin{pmatrix}
        \dot{x}_{w_{i,k}} \\ \dot{y}_{w_{i,k}} 
    \end{pmatrix}
    &= \dot{\textbf{p}}_1 
        + \left( \sum_{m=1}^{i-1} R\left(\psi_m  + \frac{\pi}{2}\right) \left( \textbf{h}_{m,r} - \textbf{h}_{m,f} \right) \dot{\psi}_m   \right) \\
        & \ \ \ \ +  R\left( \psi_{i}  + \frac{\pi}{2} \right) \left(  \textbf{w}_{i,k}\vert_{\{i\}} - \textbf{h}_{i,f} \right) \dot{\psi}_{i} 
\end{align*}
Substituting the above result in equation \eqref{nonholonomic_wheel} and by using equation \eqref{eq:SO2_sequence}, the following result is obtained:
\begin{align*}
    &\begin{pmatrix} v_{w_{i,k}}  \\ 0 \end{pmatrix} 
    = R(\psi_i +\theta_{i,k} )^\top \dot{\textbf{p}}_1 \\
        & \ \ \ \ + \left( \sum_{m=1}^{i-1} R\left(\psi_m  -\psi_i -\theta_{i,k}  + \frac{\pi}{2}\right) \left( \textbf{h}_{m,r} - \textbf{h}_{m,f} \right) \dot{\psi}_m   \right) \\
        & \ \ \ \ \ \ \ \ +  R\left( \frac{\pi}{2} -\theta_{i,k}  \right) \left(\textbf{w}_{i,k}\vert_{\{i\}} - \textbf{h}_{i,f} \right) \dot{\psi}_{i}  
\end{align*}
Hence, in matrix form, the wheel velocity and the nonholonomic constraint on wheel $\mathcal{W}_{i,k}$ is given as follows:
\begin{align}
    \begin{pmatrix} v_{w_{i,k}}  \\ 0 \end{pmatrix} 
    &= \begin{bmatrix}
         V_{i,k}\left( \textbf{q} \right) \\  
         A_{i,k}\left( \textbf{q} \right)
    \end{bmatrix}
    \dot{\textbf{q}} \nonumber \\ 
    &=
    \begin{bmatrix}
        R(\psi_i + \theta_{i,k})^\top &
        \mathcal{I}_{i,k} &
        \mathcal{S}_{i,k} &
        \textbf{0}_{2\times n-i} & 
        \textbf{0}_{2 \times \sum{K_i}}
    \end{bmatrix}
    \dot{\textbf{q}} 
     \label{eq:wheel_velocity_constraint}
\end{align}

where, $\forall 1 \leq i\leq n$ and $\forall 1<m<i$
\begin{align*}
    \mathcal{I}_{i,k} &=  \begin{bmatrix} 
                        \mathcal{I}_{i,2,k}   & 
                        \dots               &
                        \mathcal{I}_{i,i-1,k}
                    \end{bmatrix} \\
    \mathcal{I}_{i,m,k} &:= R\left(\psi_m  -\psi_i -\theta_{i,k}  + \frac{\pi}{2}\right) \left( \textbf{h}_{m,r} - \textbf{h}_{m,f} \right)  \\ 
    \mathcal{S}_{i,k} &:= R\left( \frac{\pi}{2} -\theta_{i,k}  \right) \left(\textbf{w}_{i,k}\vert_{\{i\}} - \textbf{h}_{i,f} \right) \\
    \dot{\textbf{q}} &= 
            \begin{bmatrix}
                \dot{\textbf{p}}_1 &
                \dot{\Psi}_{i-}^\top &
                \dot{\psi}_i &
                \dot{\Psi}_{i+}^\top & 
                \dot{\Theta}^\top
            \end{bmatrix}^\top \\
    \dot{\Psi}_{i-} &=
            \begin{bmatrix}
                \dot{\psi}_{1} & 
                \dot{\psi}_{2} & 
                \dots &
                \dot{\psi}_{i-1}
            \end{bmatrix}^\top \\
    \dot{\Psi}_{i+} &=
            \begin{bmatrix}
                \dot{\psi}_{i+1} & 
                \dot{\psi}_{i+2} & 
                \dots &
                \dot{\psi}_{n}
            \end{bmatrix}^\top \\
    \dot{\Theta} &=
            \left[
            \begin{array}{ccc|c|ccc}
                \dot{\theta}_{1,1} &
                \dots &
                \dot{\theta}_{1,K_1} & 
                \dots & 
                \dot{\theta}_{n,1} & 
                \dots &
                \dot{\theta}_{n,K_n}
            \end{array}
            \right]
\end{align*}


In matrix equation \eqref{eq:wheel_velocity_constraint}, the top row 
describes the transformation from the system's generalized velocity to the wheel velocity, while the bottom row $A_{i,k}(\textbf{q})\dot{\textbf{q}}=0$ describes a nonholonomic constraint on wheel $\mathcal{W}_{i,k}$ in Pfaffian form \cite{lynch_park_robotics}.
The component $\mathcal{I}_{i,m,k}$ accounts for the impact of yaw rates of all intermediate vehicle units (that is, $\mathcal{V}_1,\dots, \mathcal{V}_{i-1}$) on the nonholonomic constraint on a wheel $\mathcal{W}_{i,k}$ located on vehicle unit $\mathcal{V}_i$. The component $\mathcal{S}_{i,k}$ accounts for the impact of yaw rate of the vehicle unit $\mathcal{V}_i$ on the nonholonomic constraint on wheel $\mathcal{W}_{i,k}$ which is located on $\mathcal{V}_i$.

\section{Model Construction: Kernel Computation}
\label{sec:kernel}
The nonholonomic constraints $A_{i,k}(\textbf{q})\dot{\textbf{q}}=0$ (equation \eqref{eq:wheel_velocity_constraint}) on all wheels can be collected in the form of a matrix equation  $A(\textbf{q})\dot{\textbf{q}} = \mathbf{0}$, where $A(\textbf{q})$ is defined as follows:
\begin{equation}
    A(\textbf{q}) = \begin{pmatrix}
                        A_1(\textbf{q}) \\ A_2(\textbf{q}) \\ \vdots \\ A_n(\textbf{q})
                    \end{pmatrix}  \ , \  
    \text{where } A_{i}(\textbf{q}) = \begin{pmatrix}
                        A_{i,1}(\textbf{q}) \\ A_{i,2}(\textbf{q}) \\ \cdots \\ A_{i,K_i}(\textbf{q})
                    \end{pmatrix} \label{eq:pfaffian}
\end{equation}

The kinematic model of the vehicle system can be constructed by computing the kernel or null space of the Pfaffian constraint matrix $A(q)$ as shown in equation \eqref{eq:kernel_kinematics}.
\begin{align}
    A(\textbf{q})\dot{\textbf{q}} &= \mathbf{0} 
    \iff 
    \dot{\textbf{q}} \in \mathrm{Ker}\,A(\textbf{q}) 
                := \left\{ \dot{\textbf{q}} \ \vert \ A(\textbf{q}) \dot{\textbf{q}} = \mathbf{0} \right\} \label{eq:kernel_kinematics} \\
    &\implies \dot{\textbf{q}} = \sum_{m=1}^{M} J_m(\textbf{q})u_{\textbf{q}_m} = J(\textbf{q}) u \label{eq:kinematic_model_jacobian} \\
    &\text{where, } J(\textbf{q}) = \begin{bmatrix}
                                     J_1(\textbf{q}) & J_2(\textbf{q}) & \cdots & J_M(\textbf{q})
                                   \end{bmatrix}
                    , \ u \in \mathbb{R}^M
    \nonumber
\end{align}
 
$J(\textbf{q})$ is a matrix whose columns $J_m(\textbf{q}), \ m \in \{1,\cdots,M\}$ are the basis vectors of $\text{Ker}\,A(\textbf{q})$, and $u$ are the kinematic controls that can be used to move the system in the directions allowed by the natural nonholonomic constraints of the system.

From the construction of the constraint equations \eqref{eq:wheel_velocity_constraint} and \eqref{eq:pfaffian}, it is evident that the constituent blocks of the matrix $A(\textbf{q})$ constructed by considering the bottom row from equation \eqref{eq:wheel_velocity_constraint} have the following structure as shown below:
\small

\begin{equation}
    \renewcommand{\arraystretch}{1.2}
    A_1(\textbf{q}) =
    \begin{bmatrix}
        \begin{pmatrix}0 & 1\end{pmatrix} R(\psi_1+\theta_{1,1})^\top    & \begin{pmatrix}0 & 1\end{pmatrix} \mathcal{S}_{1,1}   & \textbf{0} \\
        \begin{pmatrix}0 & 1\end{pmatrix} R(\psi_1+\theta_{1,2})^\top    & \begin{pmatrix}0 & 1\end{pmatrix} \mathcal{S}_{1,2}   & \textbf{0} \\
        \vdots                                                           & \vdots              & \vdots      \\
        \begin{pmatrix}0 & 1\end{pmatrix} R(\psi_1+\theta_{1,K_1})^\top  & \begin{pmatrix}0 & 1\end{pmatrix} \mathcal{S}_{1,K_1} & \textbf{0} \\
    \end{bmatrix}
    \label{eq:A1_block}
\end{equation}
\begin{equation}
    \renewcommand{\arraystretch}{1.2}
    \underset{\text{for } i>1}{A_i(\textbf{q})} = 
    \begin{bmatrix}
        \begin{pmatrix}0 & 1\end{pmatrix}R(\psi_i+\theta_{i,1})^\top   & \begin{pmatrix}0 & 1\end{pmatrix}\mathcal{I}_{i,1}   & \begin{pmatrix}0 & 1\end{pmatrix}\mathcal{S}_{1,1}   & \textbf{0} \\
        \begin{pmatrix}0 & 1\end{pmatrix}R(\psi_i+\theta_{i,2})^\top   & \begin{pmatrix}0 & 1\end{pmatrix}\mathcal{I}_{i,2}   & \begin{pmatrix}0 & 1\end{pmatrix}\mathcal{S}_{1,2}   & \textbf{0} \\
        \vdots                        & \vdots              & \vdots    & \vdots      \\
        \begin{pmatrix}0 & 1\end{pmatrix}R(\psi_i+\theta_{i,K_1})^\top & \begin{pmatrix}0 & 1\end{pmatrix}\mathcal{I}_{i,K_1} & \begin{pmatrix}0 & 1\end{pmatrix}\mathcal{S}_{1,K_1} & \textbf{0} \\
    \end{bmatrix}
    \label{eq:Ai_block}
\end{equation}

\normalsize
Note that, in equations \eqref{eq:A1_block} and \eqref{eq:Ai_block}, the matrix $\begin{pmatrix} 0 & 1 \end{pmatrix}$ is used as a selection operator that selects the bottom row from equation \eqref{eq:wheel_velocity_constraint} to achieve constraint construction for the entire rigid body and the vehicle system in accordance with equation \eqref{eq:pfaffian}. 

The kernel of the matrix $A(\textbf{q})$ can be computed by transforming the equation $A(\textbf{q})\dot{\textbf{q}}$ such that an upper triangular pattern of zeros emerges, thus making an iterative solution by back substitution feasible. This is achieved by first placing the coordinate frames ${i}$ on the first wheel (arbitrarily indexed at $k=1$) mounted on vehicle unit $\mathcal{V}_i$, thus making $\mathcal{S}_{i,1}=0, \, \forall i$. Then, in order to create a upper triangular pattern of zeros, all rows except the first row and another arbitrarily selected row in matrix $A_1(\textbf{q})$ (equation \eqref{eq:A1_block}) are eliminated. Similarly, all rows except the first row in matrix $A_i(\textbf{q}), \, \forall i>1$ (equation \eqref{eq:Ai_block}) are eliminated. These row eliminations have a physical implication that imposes the existence of an instantaneous center of rotation for each moving rigid chassis. As a result, only two wheels on the tractor and one wheel on the trailer can be independently commanded steering inputs.
Hence, after eliminating rows, $\textbf{x} \subseteq \textbf{q}, \, x\in \mathcal{C}_r$ are coordinates in the reduced configuration manifold $\mathcal{C}_r$ for which kernel computation is feasible. 
\begin{equation*}
 \textbf{x} = [
    \begin{underbrace}{x_1 , y_1}_{\mathcal{V}_1\text{ location}} \end{underbrace} ,
    \begin{underbrace}{\psi_1 , \psi_2 , \cdots , \psi_n}_{\text{chassis orientation}}\end{underbrace} , 
    \begin{underbrace}{\theta_{1,1} , \theta_{1,K_1} , \theta_{2,K_2} \cdots , \theta_{n,K_n}}_{\text{orientation of independent wheels}}\end{underbrace}] \in \mathcal{C}_r
\end{equation*}
Hence, the iterative kernel computation (from first to last row) of the system of equations $A(\textbf{x})\dot{\textbf{x}}=0$ gives the kinematic model $\dot{\textbf{x}} = J(\textbf{x}) u$, where $u = \begin{bmatrix} v_{w_{1,1}} & \omega_{1,1} & \omega_{1,k} & \omega_{2,1} & \cdots & \omega_{n,1} \end{bmatrix}^\top$.
%
%
The kinematic model obtained from the iterative method is of the form shown in equation \eqref{eq:kinematic_model}
\begin{equation}
    \dot{\textbf{x}} = \left(\begin{array}{c|c}
                                \mathcal{F}(\textbf{x})  & \textbf{0}_{(n+2) \times \sum_i K_i} \\ \hline
                                \textbf{0}_{\sum_i K_i \times 1}  & I_{\sum_i K_i}
                             \end{array}\right) u
    \label{eq:kinematic_model}
\end{equation}
where, $\mathcal{F}(\textbf{x}) := \begin{bmatrix} f_{x_1}(\textbf{x}) & f_{y_1}(\textbf{x}) & f_{\psi_1}(\textbf{x}) & \cdots & f_{\psi_n}(\textbf{x}) \end{bmatrix}^\top$, such that $f_{x_1}(\cdot): \mathcal{C}_r \rightarrow \mathbb{R}$, $f_{y_1}(\cdot): \mathcal{C}_r \rightarrow \mathbb{R}$, and $f_{\psi_i}(\cdot): \mathcal{C}_r \rightarrow \mathbb{R}$, for $i \in {1,2,\cdots,n}$ are scalar-valued functions.

This iterative kernel computation will be explored in greater detail in Section \ref{sec:genAckermann}.

\subsection{Rearward Yaw Rate Amplification: }
The rearward yaw rate amplification of $\mathcal{V}_j$ with respect to $\mathcal{V}_i$, denoted by $\mathrm{RWA}_{i,j}$ is calculated as shown in equation \eqref{eq:rwa}.
\begin{equation}
    \mathrm{RWA}_{i,j} = \frac{\dot{\psi}_j}{\dot{\psi}_i} 
                       = \frac{f_{\psi_i}(\textbf{x}) v_{w_{1,1}}}{f_{\psi_j}(\textbf{x}) v_{w_{1,1}}}
                       = \frac{f_{\psi_i}(\textbf{x})}{f_{\psi_j}(\textbf{x})}
    \label{eq:rwa}
\end{equation}
In the literature, rearward amplification is defined as the ratio between the peak values of lateral acceleration and/or yaw rates of a trailer unit with respect to the tractor unit \cite{trigell_truck_trailer_dynamics,winkler1999rollover,luijten_rwa}. Rearward amplification is a good metric to analyze directional stability since it directly correlates with the rollover propensity of an articulated vehicle \cite{winkler1999rollover,ahmadian_rwa,rwa_lqr}. In this paper, the rearward yaw rate amplification is of interest from an analysis point of view since yaw rates being generalized velocities of the system, can be directly obtained from these first-order kinematic models. Thus, understanding the predictive capability of these first order models for rearward yaw rate amplification response metrics is of profound importance for motion planning and control design, especially for safety assurance at the trajectory and motion planning level in a directional and rollover safety sense.

\section{Symbolic Algorithm: Summary and Computational Complexity Analysis}
\label{app:code}
\RestyleAlgo{ruled}
The symbolic algorithm for deriving kinematic models for generalized $n$-trailer multi-axle articulated vehicles can be broken down into three broad steps:
\begin{enumerate}
    \item Define symbolic variables for generalized coordinates and parameters: Refer Algorithm \ref{alg:sub1_declareGenCoords}.
    \item Construction of nonholonomic constraint matrix $A(\textbf{x})$ for the system: Refer Algorithm \ref{alg:constraintComputation}.
    \item Constructing the forward kinematics model by evaluating $\mathrm{Ker}A(\textbf{x})$: Refer Section \ref{sec:kernel}.
\end{enumerate}

\begin{algorithm}
    \caption{Subroutine: Define symbolic variables for generalized coordinates and vehicle unit geometry}
    \label{alg:sub1_declareGenCoords}
    \DontPrintSemicolon
    \KwIn{Number of vehicle units $n$ with number of wheels $[K_1, K_2, \cdots , K_n]$, where vehicle unit $\mathcal{V}_i$ has $K_i$ wheels, $i \in \{1,2,\cdots,n\}$.}

    \SetKwFunction{algo}{algo}\SetKwFunction{proc}{defineSymbolicVariables}
    
    \SetKwProg{myproc}{Procedure}{}{}
        \myproc{\proc{$n,\mathcal{K}$}}{
        \nl Define tractor location $x_1, y_1$\; 
        \nl $\textbf{q} \gets [x_1, y_1]$ \tcp*{Generalized coordinates}
    
        \nl \For{$i \gets 1$ \KwTo $n$} {
            \nl Define symbolic variable $\psi_i$ \; 
            \nl Append $\psi_i$ to $\textbf{q}$
        }
        
        \nl $\textbf{x} \gets \textbf{q}$ \tcp*{Reduced configuration for feasible kernel computation}
        
        \nl \For{$i \gets 1$ \KwTo $n$} {
            \nl \For{$k \gets 1$ \KwTo $K_i$} {
                \nl Define steering angle $\theta_{i,k}$\;
                \nl Append $\theta_{ik}$ to $\textbf{q}$\;
                \nl \If{($i == 1 \land k \leq 2$) \textbf{or} ($i > 1 \land k < 2$)} {
                    \nl Append $\theta_{i,k}$ to $\textbf{x}$ 
                }
            }
        }
        \nl \KwRet $\textbf{x}$ \;
        }
\end{algorithm}

\begin{algorithm}
    \caption{Construction of nonholonomic constraint matrix}
    \label{alg:constraintComputation}
    \DontPrintSemicolon
    \KwIn{Number of vehicle units $n$ with number of wheels $\mathcal{K} = [K_1, K_2, \cdots , K_n]$.}

    \SetKwFunction{algo}{algo}\SetKwFunction{proc}{defineSymbolicVariables}
    
    \SetKwProg{myalg}{Algorithm}{}{}
    \myalg{\algo{$n,\mathcal{K}$}}{
    \nl $\textbf{x} \gets$ \proc($n,\mathcal{K}$) \;
    \nl \For{$i \gets 1$ \KwTo $n$} {
    \nl       Define symbolic hitch positions $h_{i,r}, h_{i,f}$\;
    \nl       \For{$k \gets 1$ \KwTo $K_i$} {
    \nl       Define symbolic wheel location $w_{i,k}\vert_{\{i\}} $\;
            }
        }
    }
    
    \nl $initCounter \gets 1$\;

    \nl \For{$i \gets 1$ \textbf{to} $n$}{
    \nl    \For{$k \gets 1$ \textbf{to} $K_i$}{
    \nl        \eIf{$(i == 1 \land k \leq 2) \lor (i>1 \land k<2)$}{
                     \tcp{Construct constraint matrix $C_{i,k}$ for independently controllable wheel $\mathcal{W}_{i,k}$}
                        \nl    $C_{i,k} \gets R(\psi_i + \theta_{i,k})^\top$ \; 
                        \tcp{where $R(\cdot):[0,2\pi)\rightarrow \mathrm{SO}(2)$}

                        \nl \If{$i > 1$}{
                        \nl    \For{$m \gets 1$ \textbf{to} $i-1$}{
                        \nl        $I_{i,m,k} \gets R\left(\psi_m - \psi_i - \theta_{i,k} + \frac{\pi}{2}\right)(h_{m,r} - h_{m,f})$ \tcp*{Constraint component due to imtermediate trailers}
                        \nl        $C_{i,k} \gets \begin{bmatrix}C_{i,k} & I_{i,m,k}\end{bmatrix}$\;
                            }
                        }

                        \tcp{Constraint Component due to vehicle unit on which wheel is located}
                        \nl $S_i \gets R\left(\frac{\pi}{2} - \theta_{ik}\right)(w_{i,k}\vert_{\{i\}} - h_{i,f})$\;
                        \nl $C_{i,k} \gets \begin{bmatrix}C_{i,k} & S_i \end{bmatrix}$\;

                        \tcp{zero padding}
                        \nl $C_{i,k} \gets \begin{bmatrix}C_{i,k} & \textbf{0}_{2,\lvert\textbf{x}\rvert - \texttt{size}(C_{i,k},2)} \end{bmatrix}$\;

                        \tcp{select bottom row: nonholonomic constraint $A_{i,k}$}
                        \nl $A_{i,k} \gets \begin{bmatrix}0 & 1\end{bmatrix}C_{i,k}$

                        \eIf{$initCounter == 1$}{
                        \nl $A \gets A_{i,k} $ \;
                        \nl $initCounter \gets initCounter + 1$
                        }
                        {
                        \nl $A \gets \begin{pmatrix}A \\ A_{i,k} \end{pmatrix} $
                        }  
                }{}
                }
            }
\end{algorithm}

Out of the three steps described above, Step 2 and 3 have the highest computational overhead due to operations within several nested loops. Hence, the computational complexity for steps 2 and 3 is analyzed below using the Big-O notation.

\begin{itemize}
    \item \textbf{Complexity of Nonholonomic constraint matrix construction} (Algorithm \ref{alg:constraintComputation}): The algorithm has the following structure:
    \begin{enumerate}
        \item The nested loop from lines 2 to 5: Considering the maximum number of wheels on every vehicle to be $K$ in the worst case, this nested loop results in $\mathcal{O}\left( \sum_{i=1}^{n} K_i \right)$ $= \mathcal{O}\left( n K \right)$ complexity.
        \item Nested loop from lines 7 to 21 also contributes $\mathcal{O}\left( n K \right)$ towards the overall complexity.
        \item The inner loop from lines 12 to 14 can run $i$ times for some wheel, thus contributing $\mathcal{O}\left( \sum_{i=1}^{n} K_i i \right)$, which contributes $\mathcal{O}\left( n^2 K \right)$ in the worst case.
    \end{enumerate}
    Hence, considering the most dominant term from the three loops discussed above, the computational complexity of the constraint matrix construction step is $\mathcal{O}\left( n^2 K \right)$.

    \item \textbf{Complexity of Kernel computation}: After Step 2, the size of the matrix $A(\textbf{x})$ scales linearly with the number of vehicle units $n$ in the articulated vehicle system, and has a lower triangular form. Hence, this step effectively involves solving a linear system of equations $A(\textbf{x})\dot{\textbf{x}}=\textbf{0}$ using back substitution. Hence, it takes one operation to solve first row, two operations to solve second row, three operations to solve third row, and so on. This results in the kernel computation step having computational complexity of $\mathcal{O}\left( \frac{n(n+1)}{2} \right) = \mathcal{O}(n^2)$.
\end{itemize}

\section{Generalized Ackermann Steering Law}
\label{sec:genAckermann}
In order to compute the configuration $\textbf{q}$, the rows of $A(\textbf{q})$ eliminated from the matrices given in equations \eqref{eq:A1_block} and \eqref{eq:Ai_block} to achieve a feasible $\text{Ker}\,A(\textbf{x})$ 
must be accounted for as algebraic constraints. 

To illustrate the above claim, iterative kernel computation for a multi-axle vehicle unit is considered in a step-by-step manner:
Consider a multi-axle tractor unit with $K_1>2$ wheels. Hence, the generalized velocities are:
\begin{equation*}
    \dot{\textbf{q}} = \begin{bmatrix}
        \dot{x}_1 & \dot{y}_1 & \dot{\psi}_1 & \dot{\theta}_{1,1} & \dot{\theta}_{1,2} & \cdots & \dot{\theta}_{1,K_1}
    \end{bmatrix}^\top
\end{equation*}
The system follows nonholonomic constraints $A(\textbf{q})\dot{\textbf{q}}=0$. Hence, using equation \eqref{eq:kernel_kinematics}, $\dot{\textbf{q}} = J(\textbf{q})u$.

\begin{itemize}
    \item \noindent\textbf{Initialization:} Let $J(\textbf{q})$ be initialized as an identity matrix, thus implying independent control of all generalized velocities as shown below:
    \begin{align*}
        J(\textbf{q}) :&= \begin{bmatrix}
                    J_{x_1} & J_{y_1} & J_{\psi_1} & J_{\theta_{1,1}}  & \cdots & J_{\theta_{1,K_1}}
                \end{bmatrix} = I_{3+K_1}
                \\ 
        u :&= \begin{bmatrix}
             u_{x_1} & u_{y_1} & u_{\psi_1} & u_{\theta_{1,1}} & \cdots & u_{\theta_{1,K_1}}
        \end{bmatrix}^\top 
    \end{align*}
    \begin{equation*}
        \implies \dot{\textbf{q}} = J(\textbf{q}) u = u
    \end{equation*}
    
    \item \textbf{Step 1:} Pick wheel $\mathcal{W}_{1,1}$. 
    Since $A(\textbf{q})\dot{\textbf{q}} = \textbf{0} \implies A_1(\textbf{q})\dot{\textbf{q}} = \textbf{0} \implies A_{1,1}(\textbf{q})\dot{\textbf{q}} = 0$, using equation \eqref{eq:A1_block}, the nonholonomic constraint on wheel $\mathcal{W}_{1,1}$ is:
    \begin{align*}
        & A_{1,1}(\textbf{q})\dot{\textbf{q}} = A_{1,1}(\textbf{q})J(\textbf{q}) u = 0 \\
        \implies & \begin{pmatrix}0 & 1\end{pmatrix} R(\psi_1+\theta_{1,1})^\top \begin{pmatrix}
            u_{x_1} \\ u_{y_1}
        \end{pmatrix}
        + \begin{pmatrix}0 & 1\end{pmatrix} \mathcal{S}_{1,1} u_{\psi_1}
        = 0
    \end{align*}
    Here, the choice of locating the origin of vehicle body frames of reference $\{i\}$ on an arbitrary wheel becomes useful since it provides a starting point for the kernel computation. Hence, without loss of generality, it can be assumed that the origin of vehicle unit's frame of reference is placed on wheel $\mathcal{W}_{1.1}$, that is, $\textbf{w}_{1,1}\vert_{\{1\}} = \textbf{0}$. Moreover, front hitch for the tractor vehicle is a trivial notion, thus implying $\textbf{h}_{i,f}=\textbf{0}$
    
    Hence, $\mathcal{S}_{1,1} = \textbf{0}$
    \begin{align}
        &\implies
        u_{x_1} \sin{(\psi_1 + \theta_{1,1})} - u_{y_1}\cos{(\psi_1 + \theta_{1,1})} = 0 \nonumber\\
        &\implies \begin{array}{c}
                    u_{x_1} = \alpha \cos{(\psi_1 + \theta_{1,1})} \\
                    u_{y_1} = \alpha \sin{(\psi_1 + \theta_{1,1})}
                 \end{array}  
    \end{align}
    Comparing with the top row of constraint equation \eqref{eq:wheel_velocity_constraint}, and substituting $u_{x_1}$ and $u_{y_1}$ from above, 
    \begin{equation}
        v_{w_{1,1}} = \begin{pmatrix}1 & 1\end{pmatrix} R(\psi_1+\theta_{1,1})^\top 
                      \begin{pmatrix} u_{x_1} \\ u_{y_1} \end{pmatrix} = \alpha
    \end{equation}
    
    Hence, it is concluded that $\alpha = v_{w_{1,K_1}}$ is the tractor velocity control.
    Hence, it is shown that the controls $u_{x_1}$ and $u_{y_1}$ are dependent. As a result, the column vectors $J_{x_1}(\textbf{q})$ and $J_{y_1}(\textbf{q})$ cannot be linearly independent. Hence, the matrix $J(\textbf{q})$ and the kinematic controls $u$ are updated in this iteration as follows:
    \begin{align}
        J(\textbf{q}) &= 
            \renewcommand{\arraystretch}{1.2}
                \begin{pNiceArray}{c | c}
                    \begin{matrix}
                        \gamma_{x_1} \\
                        \gamma_{y_1}
                    \end{matrix}   & \mathbf{0} \\ \hline
                    \mathbf{0}     & I_{\left( 1 + K_1\right) }
                \end{pNiceArray}
            \renewcommand{\arraystretch}{1}
            \label{eq:jacobian_update_1} \\
        u &= \begin{bmatrix}
                v_{w_{1,1}} & u_{\psi_1} & u_{\theta_{1,1}} & u_{\theta_{1,2}} & \cdots & u_{\theta_{1,K_1}}   
             \end{bmatrix}^\top \nonumber \\ 
        \text{where, } & \ \ \ \gamma_{x_1} = \cos{(\psi_1 + \theta_{1,K_1})} \nonumber \\
        & \ \ \ \gamma_{y_1} = \sin{(\psi_1 + \theta_{1,K_1})} \nonumber 
    \end{align}

    \item \textbf{Step 2:} Pick wheel $\mathcal{W}_{1,2}$. Using equations \eqref{eq:wheel_velocity_constraint} and \eqref{eq:jacobian_update_1}, the nonholonomic constraint on $\mathcal{W}_{1,2}$ is:
    \begin{align*}
        & A_{1,2}(\textbf{q})\dot{\textbf{q}} = A_{1,2}(\textbf{q})J(\textbf{q}) u = 0 \\
        \implies & \begin{pmatrix}0 & 1\end{pmatrix} R(\psi_1+\theta_{1,2})^\top 
                    \begin{pmatrix}
                        v_{w_{1,1}} \gamma_{x_1} \\
                        v_{w_{1,1}} \gamma_{y_1} 
                    \end{pmatrix}  + \begin{pmatrix}0 & 1\end{pmatrix} \mathcal{S}_{1,2} u_{\psi_1}
        = 0
    \end{align*}
    Substituting $\mathcal{S}_{1,2} = R\left(\frac{\pi}{2}-\theta_{1,2}\right)\textbf{w}_{1,2}\vert_{\{1\}}$ from equation \eqref{eq:wheel_velocity_constraint}, and defining $\textbf{w}_{1,2}\vert_{\{1\}} := (a_{1,2},b_{1,2})$
    \begin{equation*}
        \implies u_{\psi_1} = v_{w_{1,1}}  \dfrac{\sin\left({\theta_{1,2} - \theta_{1,1}}\right)}
                                                { a_{1,2}\cos{\theta_{1,2}} + b_{1,2} \sin{\theta_{1,2}}
                                                }
    \end{equation*}
    Hence, the matrix $J(\textbf{q})$ and the kinematic controls $u$ are updated in this iteration as follows:
    \begin{align}
        J(\textbf{q}) &= 
                \renewcommand{\arraystretch}{1.2}
                \begin{pNiceArray}{c | c }
                            \begin{matrix}
                                \gamma_{x_1} \\ 
                                \gamma_{y_1} \\
                                \gamma_{\psi_1}
                            \end{matrix}            & \mathbf{0}   \\ \hline
                            \mathbf{0}              & I_{\left( \sum_i K_i \right) }
                \end{pNiceArray} 
                \renewcommand{\arraystretch}{1}
             \label{eq:jacobian_update_2} \\
        u &= \begin{bmatrix}
                v_{w_{1,1}} & u_{\theta_{1,1}} &  u_{\theta_{1,2}} & \cdots & u_{\theta_{1,K_1}}   
             \end{bmatrix}^\top \nonumber \\
        &\text{where, }  \ \ \ 
        \gamma_{\psi_1} = \dfrac{\sin\left({\theta_{1,2} - \theta_{1,1}}\right)}
                                { a_{1,2}\cos{\theta_{1,2}} + b_{1,2} \sin{\theta_{1,2}} } \nonumber
    \end{align}

    $\cdots$ \\ 
    $\cdots$ (continue iterations for other wheels on tractor $\mathcal{V}_1$) \\ 
    $\cdots$  

    \item \textbf{Step $k$:} Continuing iterations for the other wheels $\mathcal{W}_{1,k}$, $2 < k \leq K_1$ on vehicle unit $\mathcal{V}_1$.
    Using equations \eqref{eq:wheel_velocity_constraint} and \eqref{eq:jacobian_update_2}, the nonholonomic constraint on $\mathcal{W}_{1,k}$ is:
    \begin{align*}
        &A_{1,k}(\textbf{q})\dot{\textbf{q}} = A_{1,k}(\textbf{q})J(\textbf{q}) u \\
        &= \begin{pmatrix}0 & 1\end{pmatrix} R(\psi_1+\theta_{1,k})^\top 
                    \begin{pmatrix}
                        v_{w_{1,1}} \gamma_{x_1} \\
                        v_{w_{1,1}} \gamma_{y_1} 
                    \end{pmatrix} + \begin{pmatrix}0 & 1\end{pmatrix} \mathcal{S}_{1,k} v_{w_{1,1}} \gamma_{\psi_1}
        = 0
    \end{align*}
    
    \begin{equation*}
        \implies
        \dfrac{\tan{\theta_{1,k}} - \tan{\theta_{1,1}}}
              {\tan{\theta_{1,2}} - \tan{\theta_{1,1}}}
        =
        \dfrac{a_{1,k} + b_{1,k} \tan{\theta_{1,k}}}
              {a_{1,2} + b_{1,2} \tan{\theta_{1,2}}}
    \end{equation*}
    \begin{align}
        \implies
        &\tan{\theta_{1,k}} 
        = \frac{\mathcal{N}}{\mathcal{D}} \label{eq:generalized_ackermann_SUV} \\
        &\text{where, } \nonumber \\
        &\mathcal{N} = (a_{1,2}-a_{1,k}) \tan{\theta_{1,1}}  \nonumber \\
                        & \hspace{0.5in}
                        + ( a_{1,k} + b_{1,2} \tan{\theta_{1,1}}) \tan{\theta_{1,2}}  \nonumber \\
        &\mathcal{D} = {a_{1,2} 
                        + (b_{1,2}-b_{1,k}) \tan{\theta_{1,2}}
                        + b_{1,k} \tan{\theta_{1,1}}}  \nonumber 
    \end{align}

\end{itemize}

 In equation \eqref{eq:generalized_ackermann_SUV}, it is observed that for the nonholonomic constraints related to rolling without slipping to hold for a multi-axle tractor vehicle unit, the steering angles of intermediate axles are dependent on the steering angles of two independent wheels. 
As a result, no independent controls corresponding to the intermediate steering axles/wheels can be admitted to the system. Hence, the controls $u_{\theta_{1,k}}, \ 2<k \leq K_1$ and their corresponding columns in the matrix $J(\textbf{q})$ must be removed since these controls are not kinematically independent.
Hence, equation \eqref{eq:generalized_ackermann_SUV} is nothing but the generalized Ackermann steering condition for kinematic steering in a multi-axle single vehicle unit. Hence, the simplification (triangular matrix form) of $\mathrm{Ker} A(q)$ achieved by eliminating rows corresponding to dependent steering actuators results in no loss of generality as long as the algebraic constraints given by the generalized Ackermann steering law are respected. 

\textit{Note:} 
The above observation does not mean that the dependent steering angles are eliminated from the vector of generalized coordinates $\textbf{q}$ and generalized velocities $\dot{\textbf{q}}$, that is, it must be emphasized again that there is no reduction in the generalized coordinates $\textbf{q}$ as a result of the nonholonomic constraints. The steering angles $\theta_{1,k}, \ 2<k \leq K_1$ are dependent on the two independent steering angles $\theta_{1,1}$ and $\theta_{1,2}$, but are needed to specify the system's configuration uniquely.

The above observations leads to the following proposition:
\begin{proposition} \label{prop:}
No more than three independent kinematic controls can be admitted to a single unit vehicle.
\end{proposition}
\begin{proof}
    For a general single unit vehicle $(N=1)$ with $K_1$ wheels, the algorithm is terminated at Step $k$, and leads to the following velocity kinematics are obtained:
    \begin{equation}
        \dot{\textbf{x}} = \begin{pmatrix}
                     \dot{x}_1 \\ \dot{y}_1 \\ \dot{\psi}_1 \\ \dot{\theta}_{1,1} \\ \dot{\theta}_{1,2}
                  \end{pmatrix}
                  = J(\textbf{x})u
                  = \renewcommand{\arraystretch}{1.2}
                    \begin{pNiceArray}{c | c }
                                \begin{matrix}
                                    \gamma_{x_1} \\ 
                                    \gamma_{y_1} \\
                                    \gamma_{\psi_1}
                                \end{matrix}            & \mathbf{0}   \\ \hline
                                \mathbf{0}              & I_{2}
                    \end{pNiceArray} 
                    \renewcommand{\arraystretch}{1}
                    \begin{pmatrix}
                        v_{w_{1,K_1}} \\ \omega_{1,1} \\ \omega_{1,2}
                    \end{pmatrix}
        \label{eq:single_unit_vehicle}
    \end{equation}
    Hence, it is evident from the kinematics that despite having arbitrary number of steerable wheels, there can be at most three independent kinematic controls in a single unit vehicle since all dependent steering angles $\theta_{1,k}. \ 2<k\leq K_1$ are defined using equation \eqref{eq:generalized_ackermann_SUV}. In equation \eqref{eq:single_unit_vehicle}, $\textbf{x}$ denotes the generalized coordinates corresponding to the reduced configuration manifold $\mathcal{C}_r$ for which the kernel computation is feasible for a single unit vehicle.
\end{proof}

\begin{corollary}
    Ackermann steering condition: Consider a rigid body vehicle of wheelbase $L$ and track width $T$ with two axles, with the front axle being steerable. The vehicle is taking a left turn of radius $R$. The steering angles on the left and right steering wheels are:
    \begin{equation}
        \tan{\theta_{f,l}} = \dfrac{L}{R - \frac{T}{2}} 
        \hspace{0.2in} , \hspace{0.2in} 
        \tan{\theta_{f,r}} = \dfrac{L}{R + \frac{T}{2}}
        \label{eq:ackermann_SUV}
    \end{equation}
\end{corollary}
\begin{proof}
    The front wheels are located at coordinates in the vehicle body frame: left wheel $\mathcal{W}_{f,l}: (L,-T/2)$ and right wheel $\mathcal{W}_{f,r}: (L,T/2)$.
    The steering angle $\theta_f$ at the lumped, hypothetical front wheel located at $(L,0)$ the middle of front axle is given by: $\tan{\theta_f} = L/R$ (refer Figure \ref{fig:ackermann_car}). 
    
    \begin{enumerate}
        \item For lumped front wheel, $a_{1,1} = L$, $b_{1,1} = 0$, $\theta_{1,1} = \theta_f$.
        \item For lumped rear wheel, $a_{1,K_1} = 0$, $b_{1,K_1} = 0$, $\theta_{1,K_1} = 0$.
        \item For front left wheel, $a_{1,k} = L$, $b_{1,k} = -T/2$, $\theta_{1,k} = \theta_{f,l}$.
        \item For front right wheel, $a_{1,k} = L$, $b_{1,k} = T/2$, $\theta_{1,k} = \theta_{f,r}$.
    \end{enumerate}

    \begin{figure}
        \centering
        \includegraphics[width=0.7\linewidth]{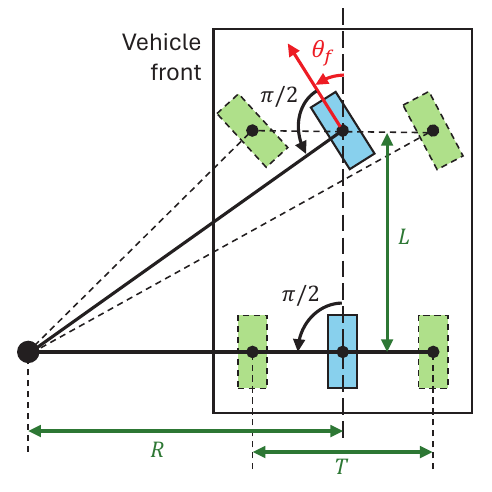}
        \caption{Ackermann steering for a car with two axles and front wheel steering}
        \label{fig:ackermann_car}
    \end{figure}
    
    By substituting the above values in equation \eqref{eq:generalized_ackermann_SUV}, the result from equation \eqref{eq:ackermann_SUV} is obtained.
\end{proof}

\subsection{Generalized Ackermann Steering: Extension to $n$-Trailers}
The iterative kernel computation procedure described for a single unit vehicle can be applied to nonholonomic constraints described in equation \eqref{eq:Ai_block} for trailers in a similar manner. 
Thus, each eliminated row in equation \eqref{eq:wheel_velocity_constraint} leads to a $n$-trailer generalization of the Ackermann steering condition \cite{genta} shown in equation \eqref{eq:generalized_ackermann}, allowing one to compute steering angle of dependent wheel located at $\textbf{w}_{i,k}\vert_{\{i\}}$ from the generalized velocities $\dot{\textbf{x}}$ obtained from kinematic model, given the independently controllable steering angles.
\begin{equation}
    \begin{bmatrix}
        0 & 1      
    \end{bmatrix} 
    \left(
        R(\psi_i + \theta_{i,k})^\top \begin{bmatrix}
            \dot{x}_1 \\ \dot{y}_1
        \end{bmatrix}
        +
        \sum_{m=1}^{i-1} \mathcal{I}_{i,m,k} \dot{\psi}_m
        +
        \mathcal{S}_{i,k} \dot{\psi}_i
    \right)
    = 0
    \label{eq:generalized_ackermann}
\end{equation}

\subsection{Virtual Steering Angles: Physical Interpretation of Generalized Ackermann Steering Law}
As a consequence of the generalized Ackermann steering law, a necessary condition for the kinematic vehicle model to be a reasonable representation of an actual vehicle is the \textit{no-slip condition}, which ensures that all points on a rigid body are rotating about an instantaneous center of rotation as illustrated in Figure \ref{fig:kinematic_vehicle}. Hence, using the generalized Ackermann Steering Law (equation \eqref{eq:generalized_ackermann}), the virtual steering angle for the can be computed for any point on the rigid vehicle unit by substituting the coordinates of a physical wheel $\textbf{w}_{i,k}\vert_{\{i\}}$ with the coordinates of any point referred to the body frame of reference.

 \begin{figure}
     \centering
     \includegraphics[trim=0cm 0cm 0cm 0cm, clip=true, width=\linewidth]{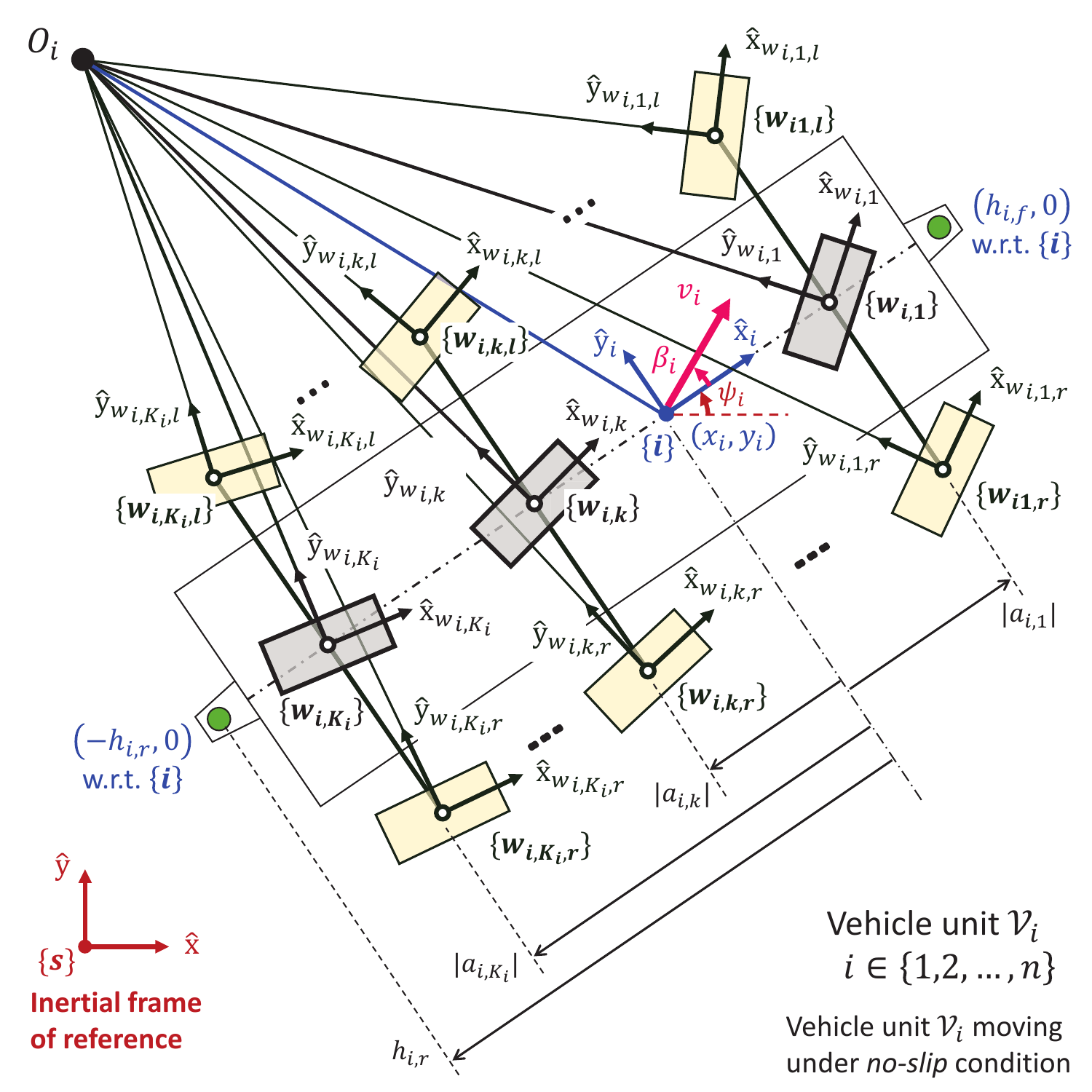}
     \caption{Illustration: Generalized Ackermann steering law for a vehicle unit: Under no-slip condition, the $\hat{\mathrm{y}}$-axes for all wheel frames of reference in a multi-axle vehicle unit $\mathcal{V}_i$ shall pass through the instantaneous center of rotation $O_i$. For a laterally symmetrical chassis unit as shown here, the bicycle approximation provides an accurate representation of the steering geometry.}
     \label{fig:kinematic_vehicle}
 \end{figure}
Tire slip is ignored for a vehicle driving under no-slip condition. Hence, the left and right wheels can be lumped into one wheel located on the vehicle's longitudinal axis to give the bicycle approximation (Figure \ref{fig:kinematic_vehicle}). The bicycle approximation is an accurate representation of the Ackermann-steered laterally symmetrical vehicle unit driving without slipping. The generalized Ackermann steering law also provides a way to determine the virtual steering angles for the lumped wheels for a general $n$-Trailer with laterally symmetrical chassis units.

The \textit{virtual hitch steering angle} is obtained by substituting $\textbf{w}_{i,k}\vert_{\{i\}}$ with hitch coordinates $\textbf{h}_{i,r}$ or $\textbf{h}_{i,f}$ in equation \eqref{eq:generalized_ackermann}. 
The hitch $H_{i-1,i}$ being part of the rigid vehicle unit $\mathcal{V}_{i-1}$ has velocity normal to the line connecting the hitch point to the instantaneous center of rotation $O_{i-1}$. As a result, the hitch can be thought of as a virtual steering that applies a steering control on the trailer vehicle unit $\mathcal{V}_i$.
As a consequence of the virtual hitch steering being defined by the preceding tractor unit, the generalized Ackermann law allows for at most one independent steering actuator kinematic control to be be admitted to a trailer unit.

\section{Model Validation and Results}
Validation of the proposed modeling approach has been performed using two different strategies outlines in Figure \ref{fig:outline_partial_experimental} and Figure \ref{fig:outline_dyn_vs_kin} respectively: 
\begin{enumerate}
    \item \textbf{Strategy 1:} (refer Figure \ref{fig:outline_partial_experimental}) The first strategy involves collecting kinematic control data from a passenger vehicle on real-world scenarios. The passenger vehicle's actual trajectory is logged using a high-precision GPS receiver. The kinematic control data is then used to study the yaw rate and trajectory of articulated vehicle configurations with virtual trailers. This approach allows a thorough analysis of model performance in real-world scenarios at a purely kinematic level.
    The trailers, being hypothetical is not detrimental to the analysis at hand because the models being kinematic in nature, only describe the velocity relationships for the various vehicle units and the hitches, and ignore any reaction forces that an actual trailer would exert on the tractor via the hitch joint. It is being assumed that the tractor's powertrain and front wheel steering is capable of providing the reference velocity and steering rate profiles logged from the onboard vehicle sensors corresponding to typical low-speed urban driving behavior, regardless of the number and mass of trailers connected to the vehicle.
    \item \textbf{Strategy 2:} (refer Figure \ref{fig:outline_dyn_vs_kin}) The second strategy involved comparing the response of the first-order kinematic model against a higher-order dynamic model of a truck-trailer system. The dynamic model being used for this validation strategy is a detailed and realistic model comprising of 6-Degree-Of-Freedom (DOF) chassis for the tractor unit, trailer unit and the hitch. Additionally, the dynamic model offers a reasonable representation of the suspension and tire behavior. This strategy allows validating the reduced order kinematic models against uncertainties and disturbances such as sensor noise, joint friction and actuator delays, which is challenging using the real-world data logged from vehicle CAN bus. One drawback of this strategy is the inability to simulate real-world scenarios involving road geometries and human driver commands.
\end{enumerate}
\begin{figure}
    \centering
    \includegraphics[width=\linewidth]{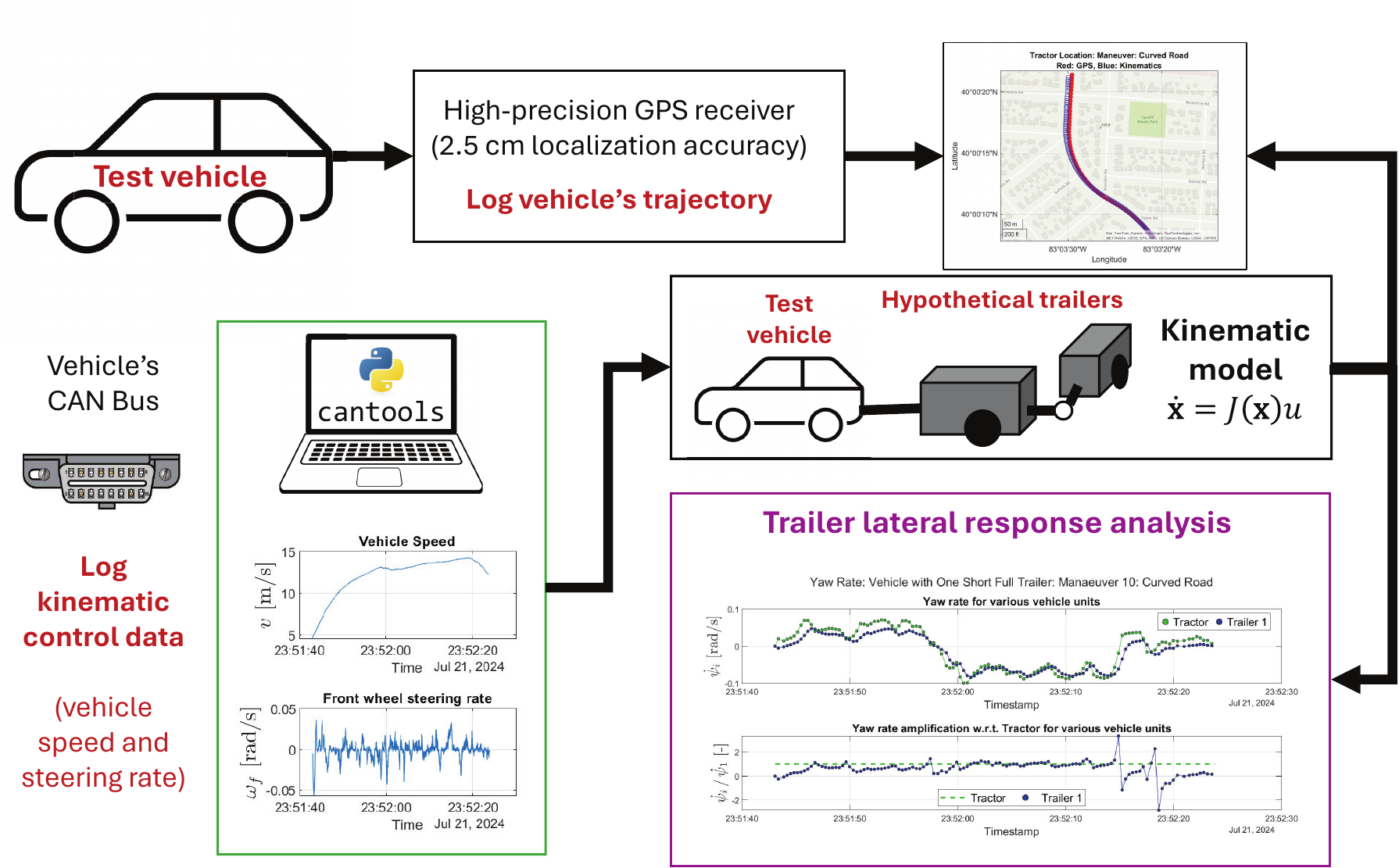}
    \caption{Partial experimental validation: Evaluating response of the kinematic model to real-world driver commands}
    \label{fig:outline_partial_experimental}
\end{figure}

\begin{figure}
    \centering
    \includegraphics[width=\linewidth]{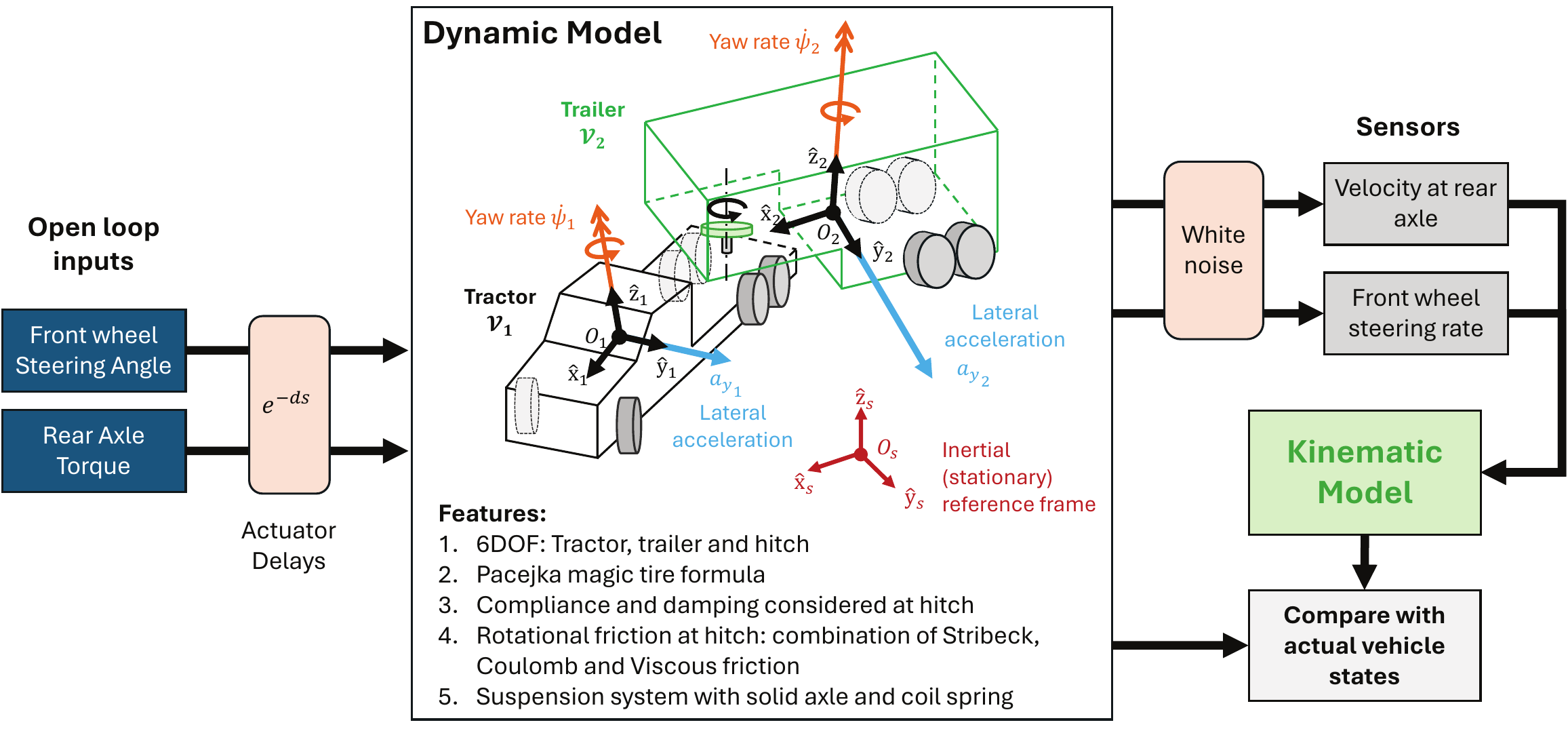}
    \caption{Outline of simulation-based validation: high-fidelity dynamic model vs reduced order kinematic model}
    \label{fig:outline_dyn_vs_kin}
\end{figure}

Detailed results for model validation using these two strategies are presented in the subsequent subsections.

\subsection{Validation Strategy 1: Partial Experimental Validation Using Response of Kinematic Models with Hypothetical Trailers to Real-World Driver Data}
Data obtained from a single-unit passenger car's (test vehicle) CAN (Controller Area Network) bus is used as input to the kinematic model in equation \eqref{eq:kinematic_model}. The following quantities are logged from the CAN bus: Steering wheel rate $\omega_f = \omega_{1,2}$, Speed at the rear axle $v$ as shown in Figure \ref{fig:stateEvolution_sc1}. The vehicle body frame of reference for the tractor unit is fixed at the rear axle, which is not steered, thus $\omega_{1,1} = 0$. The vehicle onboard sensors used in this study are:
\begin{enumerate}
    \item Hand wheel steering angle rate [deg/s] sensor's measurement is divided by vehicle steering gear train's steering ratio to obtain the steering rate at the vehicle's front wheel.
    \item Speed at the rear axle [km/h]: The vehicle's ECU computes vehicle speed internally from the wheel speed by using a nominal wheel radius. 
\end{enumerate}

\begin{figure}
    \centering
    \begin{subfigure}[b]{0.9\linewidth}
        \centering
        \includegraphics[width=0.9\linewidth]{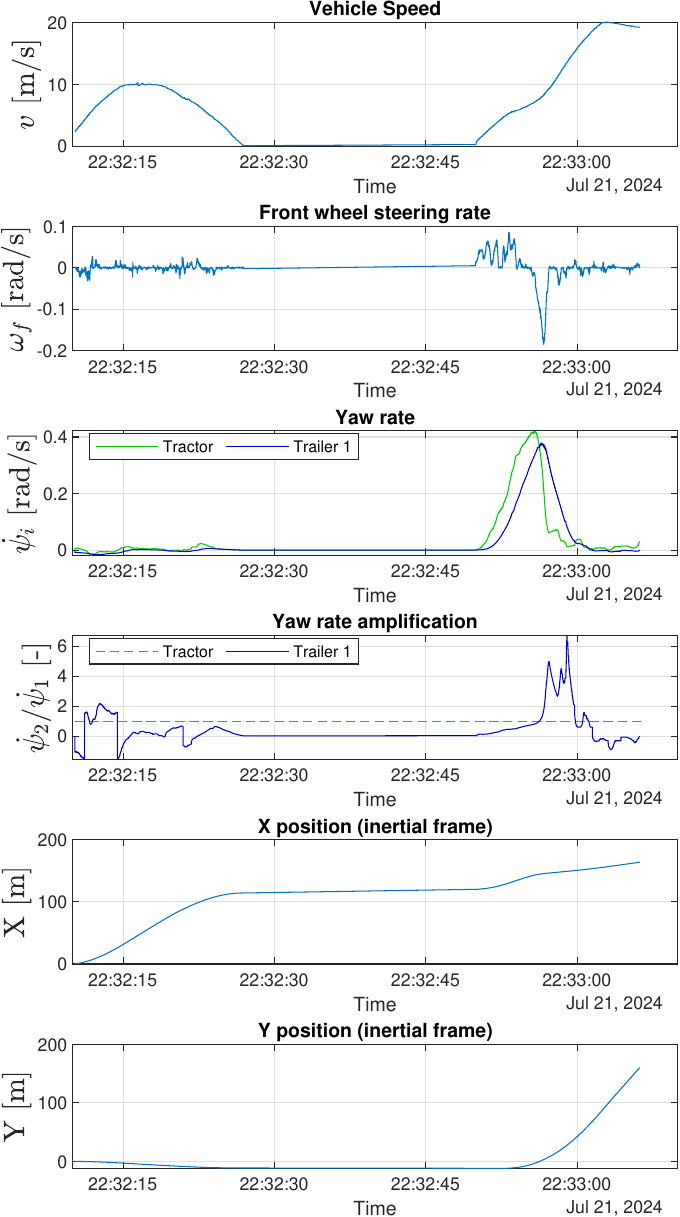}
        \caption{(Top to bottom) (1.,2.) Kinematic control data collected from single-unit test vehicle's CAN bus. Data collected: (1.) velocity of tractor's rear axle and (2.) front wheel steering rate. \\
        (3.) Yaw rate of the tractor and trailer \\ 
        (4.) Rearward yaw rate amplification response of the simulated trailer \\
        (5.,6.) Vehicle trajectory with respect to inertial frame of reference.}
        \label{fig:stateEvolution_sc1}
    \end{subfigure}
    \vskip 0.1in
    \begin{subfigure}[b]{0.9\linewidth}
        \centering
        \includegraphics[width=0.8\linewidth]{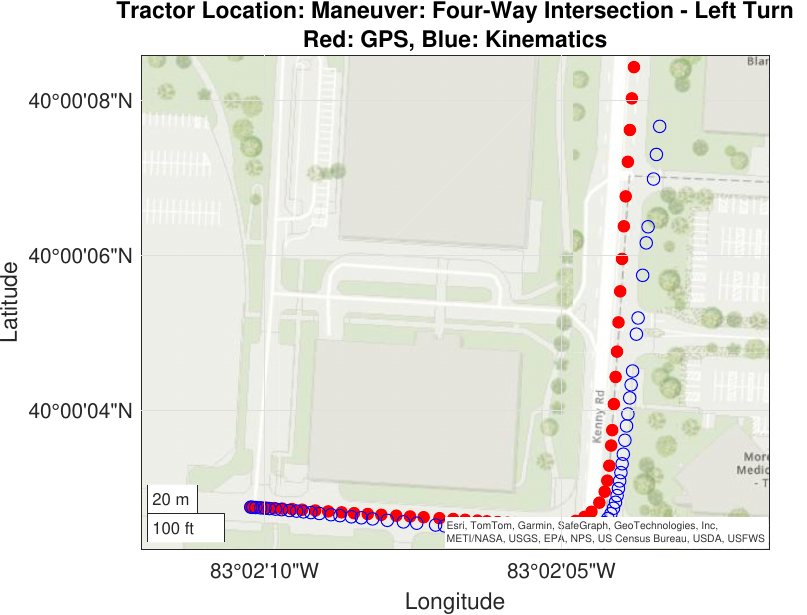}
        \caption{Path of the single-unit test vehicle computed using kinematic model as compared to its actual path obtained from GPS.}
        \label{fig:tracPos_sc1}
    \end{subfigure}
    \caption{Simulated kinematic behavior of a one-trailer vehicle system executing a left turn maneuver at a four-way intersection}
    \label{fig:test_scenario}
\end{figure}

The vehicle is driven on a test scenario with a variable  speed profile and its actual path is logged using a high-precision NovAtel PwrPak7D GPS/GNSS receiver having an accuracy of 2.5 cm. The kinematic model is solved using numerical integration and the trajectory $(x_1(t),y_1(t))$ is compared to the 
GPS trajectory in Figure \ref{fig:tracPos_sc1}. 
A drift between kinematic trajectory and GPS trajectory is noticed since tire forces and slip effects are ignored by the model. However, ignoring these effects leads to an elegant first-order model structure which is useful to understand the controllability properties by calculating Lie brackets of various columns of $J(\textbf{x})$ \cite{laumond_controllability,general_n_trailer_properties}, and for developing motion planning algorithms such as the work presented in \cite{Orosco_modeling_feedback_linearization_n_trailer,flatness_motion_planning_1,flatness_motion_planning_2}.

\begin{figure}
    \centering
    \includegraphics[width=0.9\linewidth]{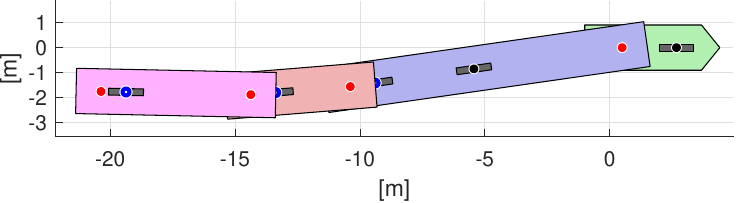}
    \caption{Test vehicle configuration simulated with three virtual trailers}
    \label{fig:vehicleConfig}
\end{figure}

\begin{figure}
        \centering
        \includegraphics[width=\linewidth]{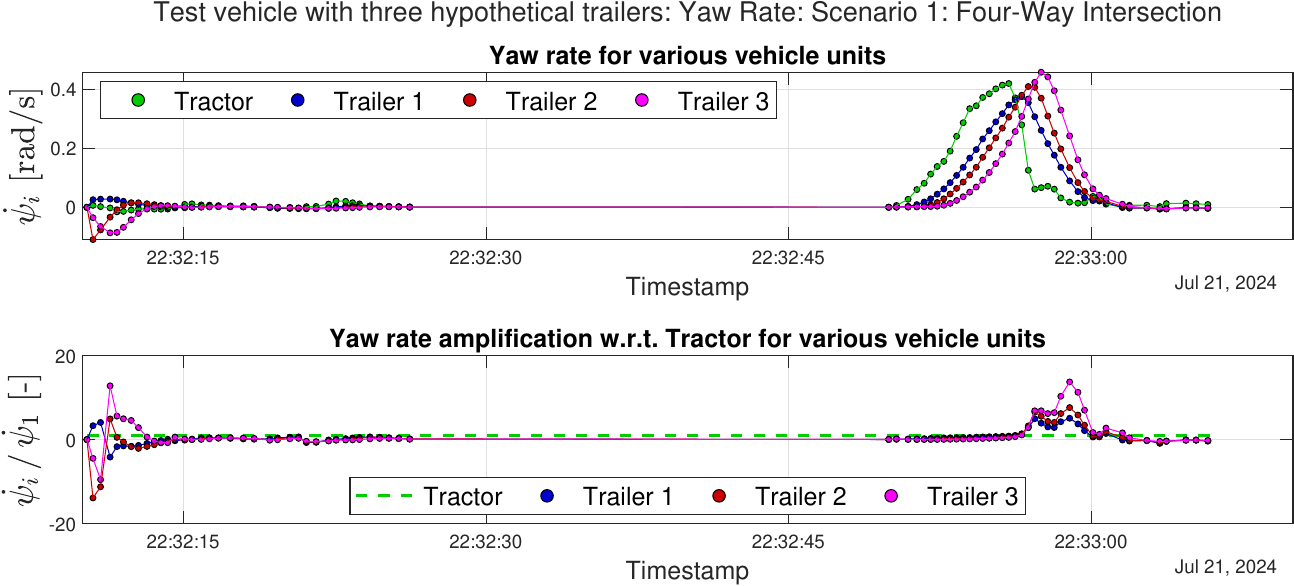}
        \caption{Three trailers}
    \caption{Test vehicle pulling three simulated trailers: Yaw rate amplification response. The lag in the yaw rate worsens as more trailers are added leading to a spike in rearward yaw rate amplification.}
    \label{fig:yawAmp}
\end{figure}

\begin{figure}
    \centering
    \begin{subfigure}[b]{0.48\linewidth}
        \centering
        \includegraphics[width=0.98\linewidth]{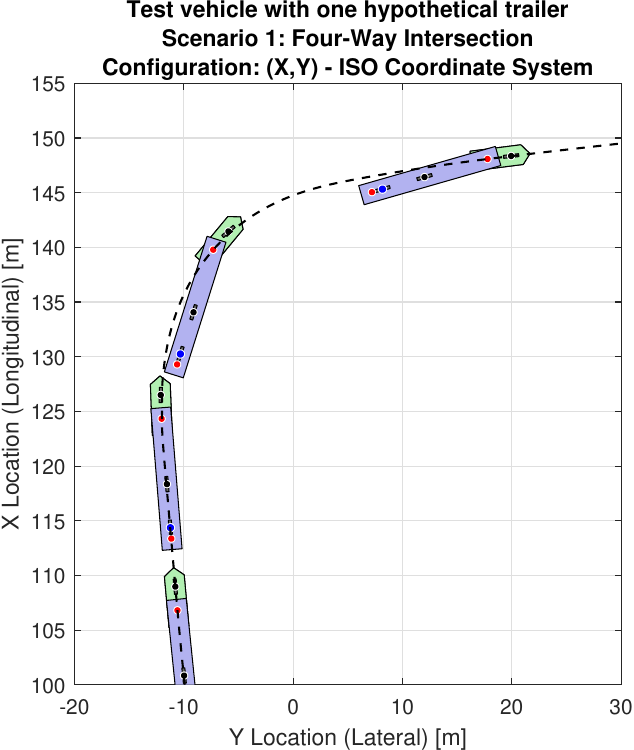}
        \caption{One trailer: The trailer deviates from the tractor's trajectory (dotted line) while taking a turn}
        \label{fig:XYPlot_1T_sc1}
    \end{subfigure}
    \hfill
    \begin{subfigure}[b]{0.48\linewidth}
        \centering
        \includegraphics[width=0.98\linewidth]{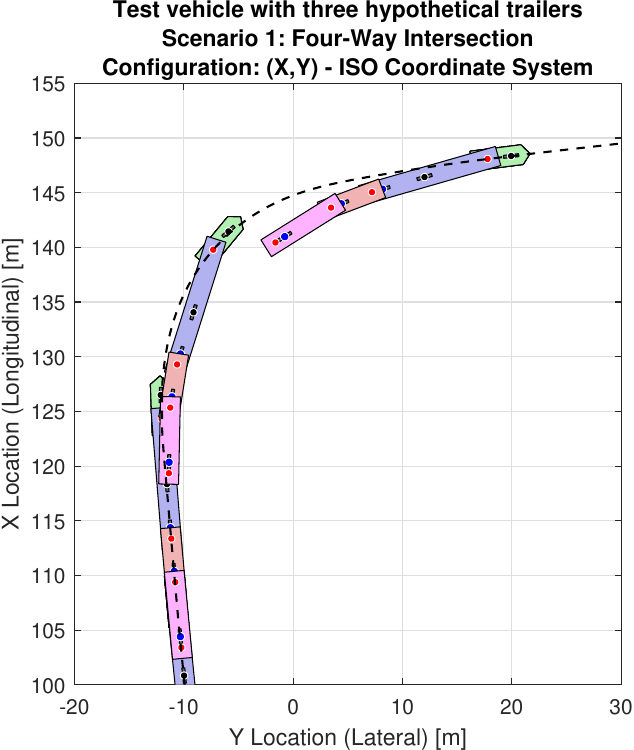}
        \caption{Three trailers: Offtracking is more pronounced for the trailers at the rear, especially while exiting the turn}
        \label{fig:XYPlot_3T_sc1}
    \end{subfigure}
    \caption{Offtracking response of the test vehicle pulling one vs. three simulated trailers. }
    \label{fig:offtracking}
\end{figure}

For the test vehicle with one simulated \textit{virtual} trailer, it is noticed in Figure \ref{fig:stateEvolution_sc1} that the yaw rate of the trailer unit lags behind the yaw rate of the tractor unit, thus leading to a spike in the rearward yaw rate amplification of the trailer with respect to the tractor. 
The test vehicle is also simulated with multiple \textit{virtual} trailers as shown in Figure \ref{fig:vehicleConfig}.
The yaw rate and rearward amplification of the trailers with respect to the tractor $\mathrm{RWA}_{1,j}$ for $j\in\{1,2,3\}$ are computed from the kinematic models and are studied in Figure \ref{fig:yawAmp}. The yaw rate of the trailing vehicle lags behind the yaw rate of the vehicle in front of it. As a result, the rearward yaw rate amplification $\mathrm{RWA}_{1,j}$ gets exacerbated as $j$ increases when more trailers are added to the vehicle system. These results for rearward amplification are consistent with the literature \cite{trigell_truck_trailer_dynamics}, thus partially validating the modeling approach.

Another kinematic behavior observed in the articulated vehicle system is that of offtracking \cite{trigell_truck_trailer_dynamics,altafini_reduced_offtracking}. Offtracking refers to the difference between the trajectories of various wheels and vehicle units in the vehicle system. It is noticed in Figure \ref{fig:offtracking} that the addition of more trailers worsens the offtracking behavior since the rearmost trailer deviates most profoundly from the desired path indicated by the dotted line.

\subsection{Validation Strategy 2: Comparison Against a High-Fidelity Dynamic Model}
\begin{figure}
    \centering
    \includegraphics[width=\linewidth]{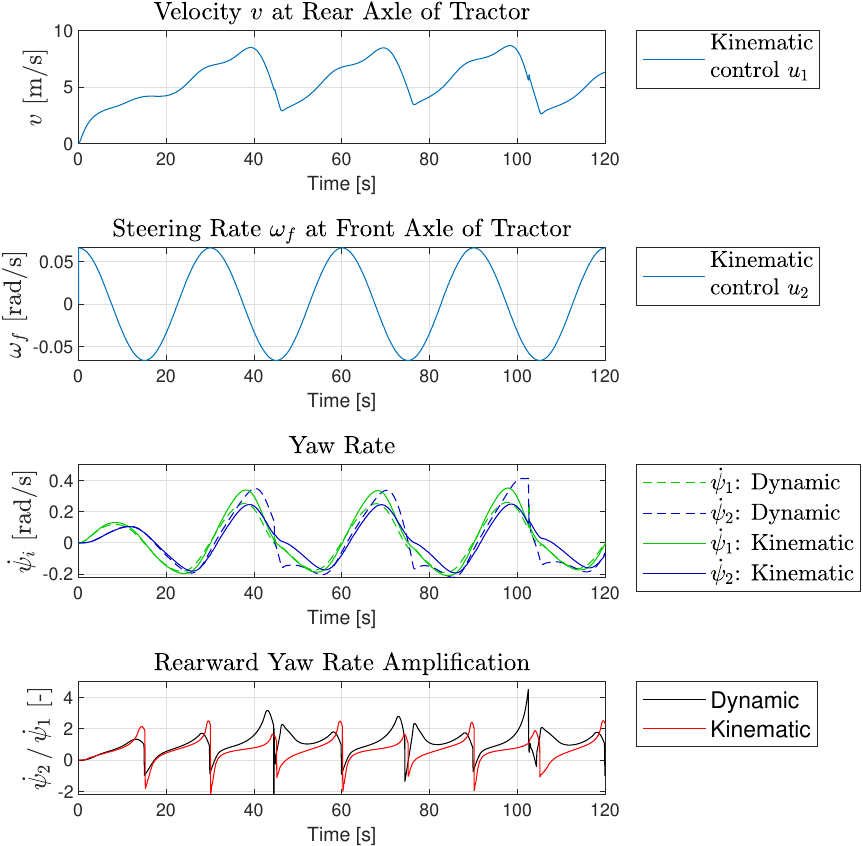}
    \caption{Dynamic model vs Kinematic model: response to open loop axle torque and sine steer inputs. The kinematic model's predicted yaw rate response is a reasonable representation at low speeds and low steering rate, but deviates from actual states under high-slip conditions.}
    \label{fig:Dyn_vs_Kin}
\end{figure}
While Validation strategy 1 gave some insight into the kinematic behavior of multi-trailer vehicles, it does not give much insight into the robustness of the kinematic models under non-ideal conditions. In order to validate the robustness and limitations of the models derived using the proposed symbolic algorithm, a high-fidelity dynamic model developed using the MathWorks Vehicle Dynamics Blockset \cite{vdbs} is utilized as described in the outline shown in Figure \ref{fig:outline_dyn_vs_kin}. The dynamic model has the following features: 

\begin{enumerate}
    \item 6-DOF chassis dynamics for tractor, trailer and hitch are considered.
    \item Tires are modeled using Pacejka magic tire formula \cite{pacejka}.
    \item The hitch forces are modeled Compliance and damping is considered for longitudinal and lateral forces on the hitch.
    \item Rotational friction at the hitch is modeled as a combination of Stribeck, Coulomb and Viscous friction \cite{armstrong_friction,MathWorks_RotationalFriction}.
\end{enumerate}

The dynamic model is given an open-loop inputs of axle torque (accelerator pedal input) to the tractor unit's powertrain, along with a sine steer reference steering angle. The dynamic model gives a realistic estimate of the velocity of the tractor at the rear axle and the steering rate at the front axle. These estimates are used by the kinematic model as control inputs in order to estimate the yaw rate response of the vehicle system. The kinematic model's estimated yaw rate response and rearward yaw rate amplification are compared against the dynamic model's realistic estimates in Figure \ref{fig:Dyn_vs_Kin}. It is noted that while the kinematic model accurately represents the yaw rate response under low-slip conditions, its estimate deteriorates under high-slip conditions. As a result, it is recommended to use kinematic models only for open-loop control problems such as motion planning and trajectory optimization, while more accurate representation of environmental factors that cause slip are ideal for stability control applications.

\subsubsection{Impact of uncertainties, joint friction, actuator delays, and sensor noise on predictive capability of the forward kinematics models}
\begin{figure}
    \centering
    \includegraphics[width=\linewidth]{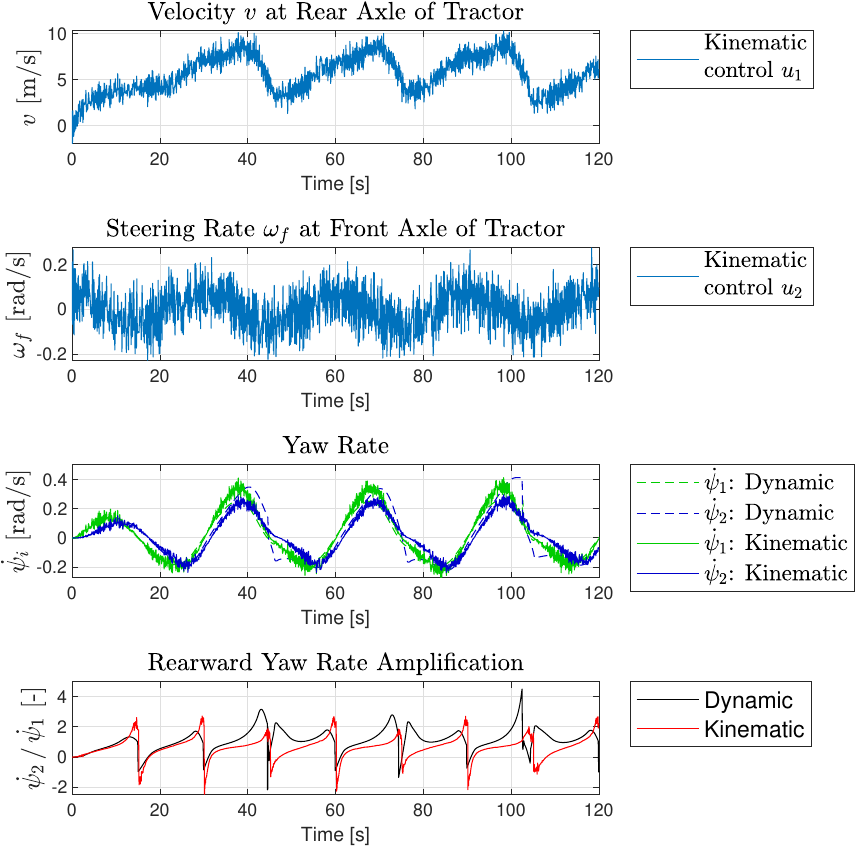}
    \caption{Dynamic model vs Kinematic model: effect of sensor noise. While the yaw rate response becomes noisier due to sensor noise, the rearward yaw rate amplification estimate is relatively noise-free.}
    \label{fig:Dyn_vs_Kin_noise}
\end{figure}
The following assumptions are generally employed when building kinematic models of vehicle systems: all vehicle chassis units are assumed to be rigid, tire slip is negligible, and friction between joints is ignored.
Due to the nature of assumptions, kinematic models are prone to losing their predictive capabilty due to environmental uncertainties like sensor noise, joint friction, tire slip, actuator delays that affect the plant system for which these models are used for control synthesis. Hence, in this subsection, an analysis is conducted to evaluate the robustness and limitations of these kinematic models, especially for representing yaw rate response.

In order to evaluate the kinematic model's robustness to sensor noise, white noise is added to the sensors that measure the kinematic controls (vehicle speed and front wheel steering rate). As observed in Figure \ref{fig:Dyn_vs_Kin_noise}, the kinematic model's estimation of the rearward yaw rate amplification is comparable to the noise-free case (Figure \ref{fig:Dyn_vs_Kin}) despite having a noisier yaw rate estimate.

\begin{figure}
    \centering
    \includegraphics[width=\linewidth]{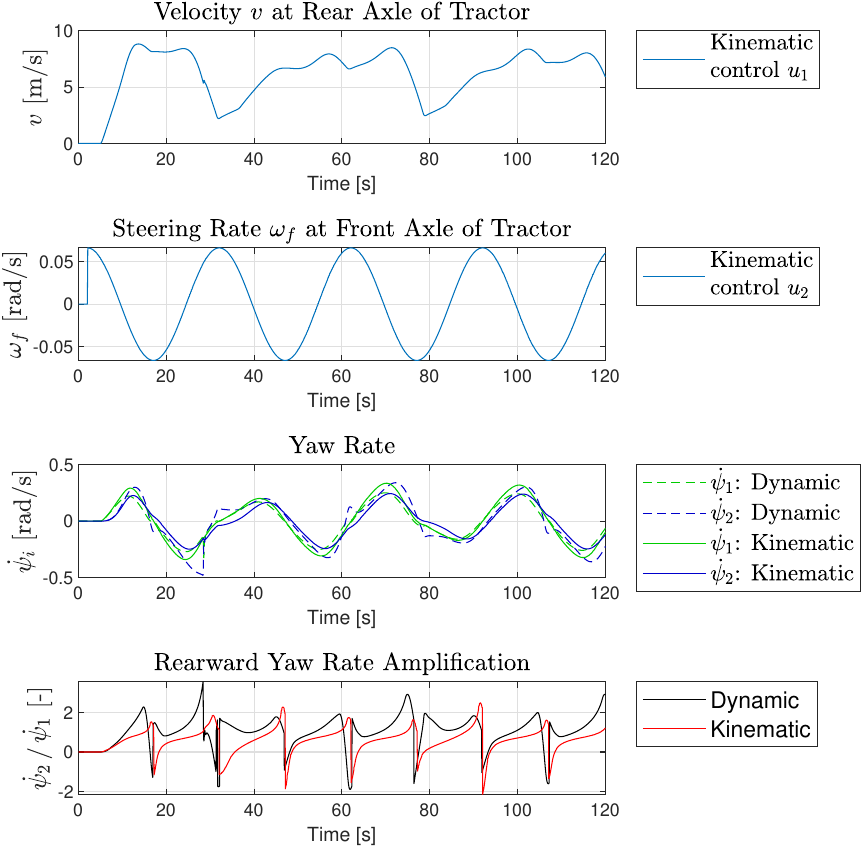}
    \caption{Dynamic model vs Kinematic model: effect of actuator delays: Torque actuator delay (5 seconds), Steering actuator delay (2 seconds) largely do not affect kinematic model's performance if the sensors that measure kinematic controls are accurate.}
    \label{fig:Dyn_vs_Kin_delay}
\end{figure}

Figure \ref{fig:Dyn_vs_Kin_delay} shows the effect of actuator delay on the kinematic model's yaw rate response estimate. In this test, the axle torque command to the dynamic model was delayed by 5 seconds, while the steering angle command to the dynamic model was delayed by 2 seconds. Despite the actuator delay, the kinematic model follows the dynamic model's yaw rate response. This can be attributed to the fact that the kinematic model uses sensors that measure the dynamic model's states (velocity and steering rate), which get affected by the actuator delay. However, the sensors being accurate in this test results in no loss of reliability of the kinematic model due to actuator delay.

The dynamic model's hitch joint is applied rotational friction using the following parameters: Coulomb friction torque (Case 1: No friction, Case 2: 450 Nm and Case 3: 2000 Nm). The breakaway friction torque is fixed at 1.2 times the Coulomb friction torque. The viscous friction coefficient is 1 Nm/rad/s for all cases. Due to joint friction, it is observed in Figure \ref{fig:Dyn_vs_Kin_friction} that the dynamic model's rearward yaw rate amplification is dampened if the friction is high. However, the kinematic model completely ignores joint friction. As a result, such models are not suitable for control design when the plant is affected by uncertainties such as hitch joint rotational friction.
\begin{figure}
    \centering
    \begin{subfigure}[b]{0.9\linewidth}
        \centering
        \includegraphics[width=0.9\linewidth]{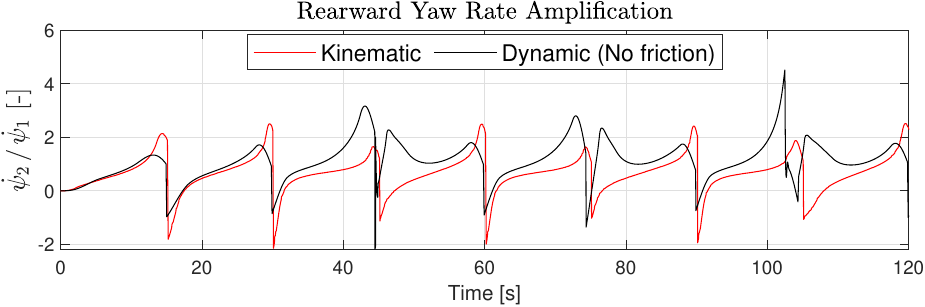}
        \label{fig:rwa_noFriction}
    \end{subfigure}
    \hfill
    \begin{subfigure}[b]{0.9\linewidth}
        \centering
        \includegraphics[width=0.9\linewidth]{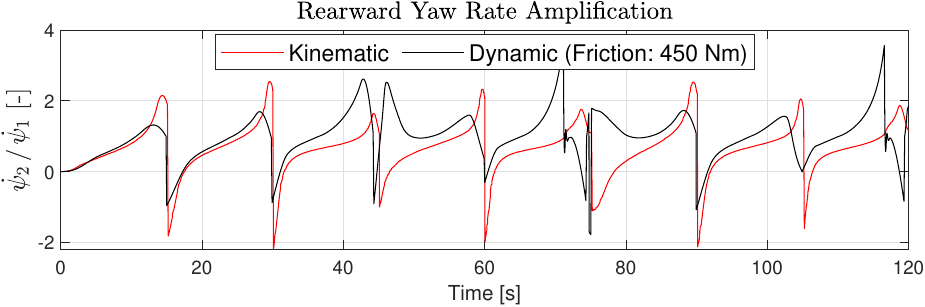}
        \label{fig:rwa_Friction_450}
    \end{subfigure}
    \hfill
    \begin{subfigure}[b]{0.9\linewidth}
        \centering
        \includegraphics[width=0.9\linewidth]{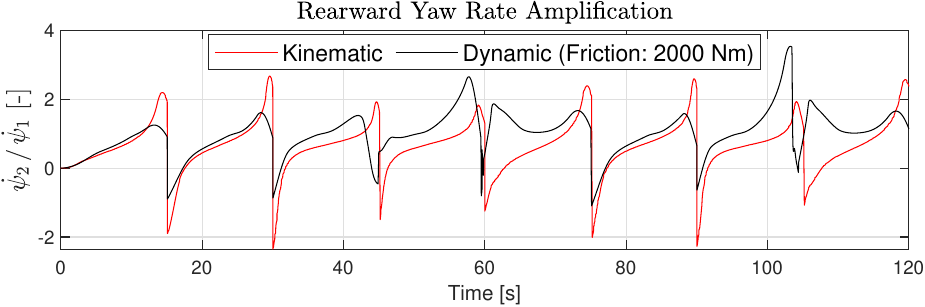}
        \label{fig:rwa_Friction_2000}
    \end{subfigure}
    \caption{Dynamic model vs Kinematic model: Impact of hitch joint friction. The kinematic model completely disregards joint friction, while the dynamic model's rearward yaw rate amplification is dampened due to higher Coulomb friction.}
    \label{fig:Dyn_vs_Kin_friction}
\end{figure}

\subsection{Kinematic Behavior of Common Vehicle Configurations}
\begin{figure}
    \centering
    \includegraphics[width=0.9\linewidth]{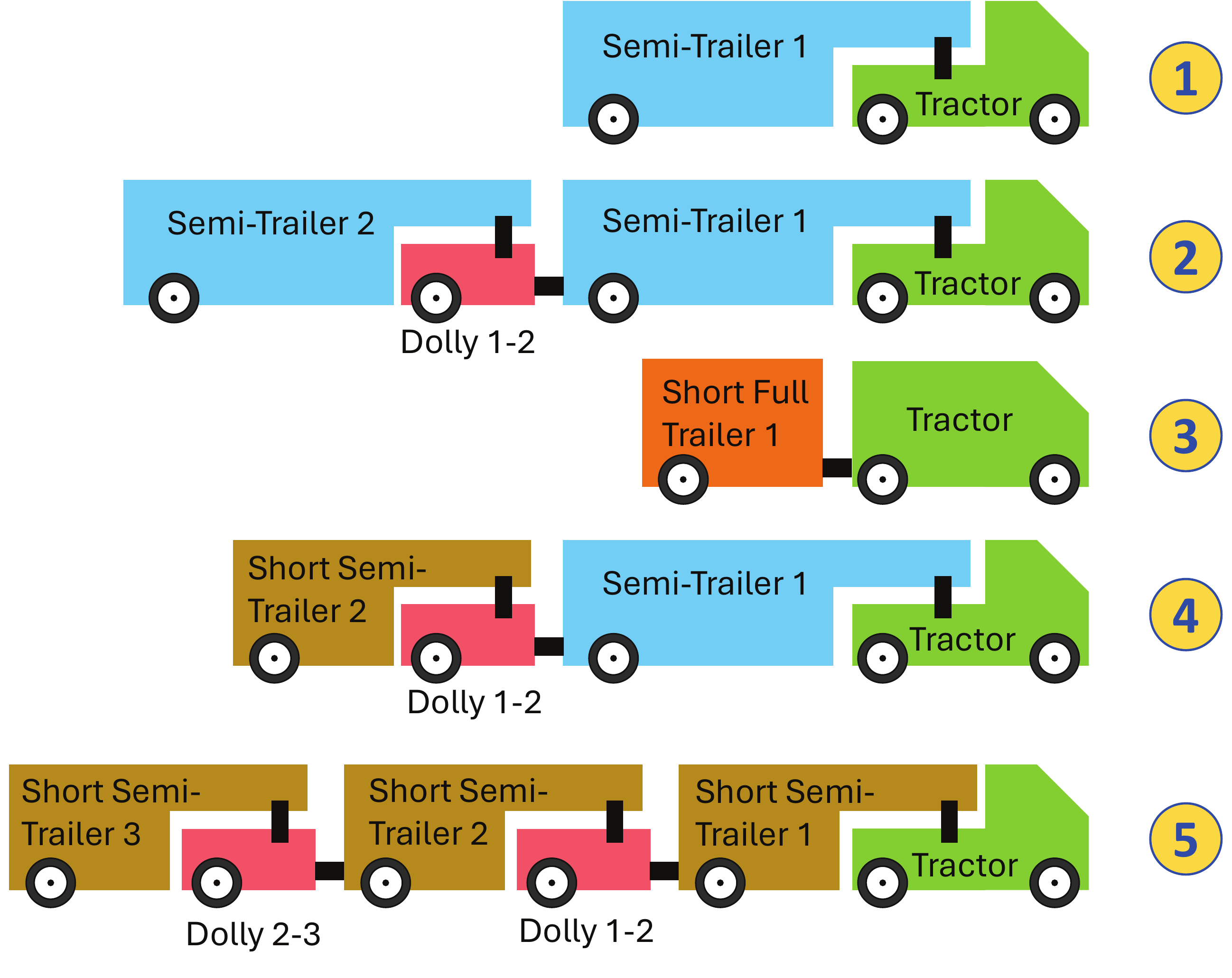}
    \caption{Common articulated vehicle configurations}
    \label{fig:configs}
\end{figure}

Some commonly used configurations of articulated vehicles are shown in Figure \ref{fig:configs}. Realistically, semi-trailers are connected to each other in multi-trailer configurations by the using converter dollies. Hence, a converter dolly must be treated as another rigid vehicle unit within the articulated vehicle system. The four configurations tested on a four-way intersection scenario are described as follows:
\begin{itemize}
    \item Configuration 1 is a system in which the tractor unit tows one semi-trailer.
    \item Configuration 2 is a system in which the tractor unit tows two semi-trailers connected to each other by a dolly unit.
    \item Configuration 3 is a system in which the tractor tows a short full trailer.
    \item Configuration 4 is a system in which the tractor unit tows one long and one short semi-trailer connected to each other by a dolly unit.
    \item Configuration 5 is a system in which the tractor unit tows three short semi-trailers connected to each other by dolly units.
\end{itemize}

All semi-trailers are connected to the vehicle unit in front by a hitch located 0.5 m in front of the rearmost axle of the unit in front, while all full trailers are connected to the unit in front through a hitch located 0.5 m behind the vehicle unit in front.  

\begin{figure}
    \centering
    \begin{subfigure}[b]{\linewidth}
        \centering
        \includegraphics[width=\linewidth]{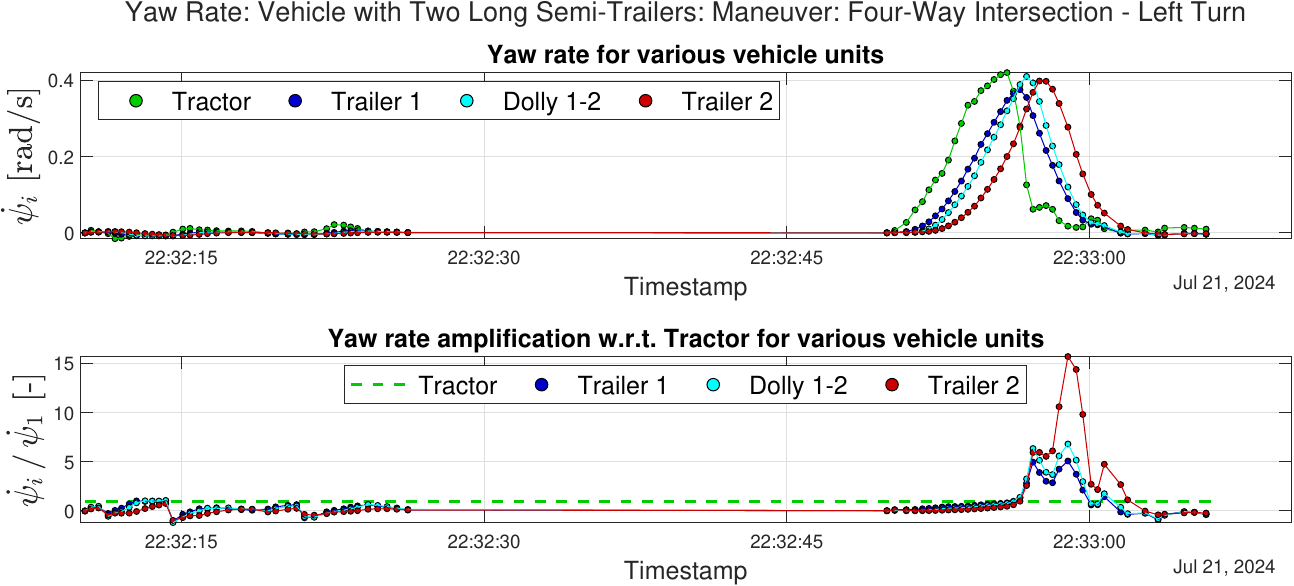}
        \caption{Configuration 2: Tractor pulling two semi-trailers}
        \label{fig:yawAmp_2T_sc1_config2}
    \end{subfigure}
    \hfill
    \vskip 0.01in
    \begin{subfigure}[b]{\linewidth}
        \centering
        \includegraphics[width=\linewidth]{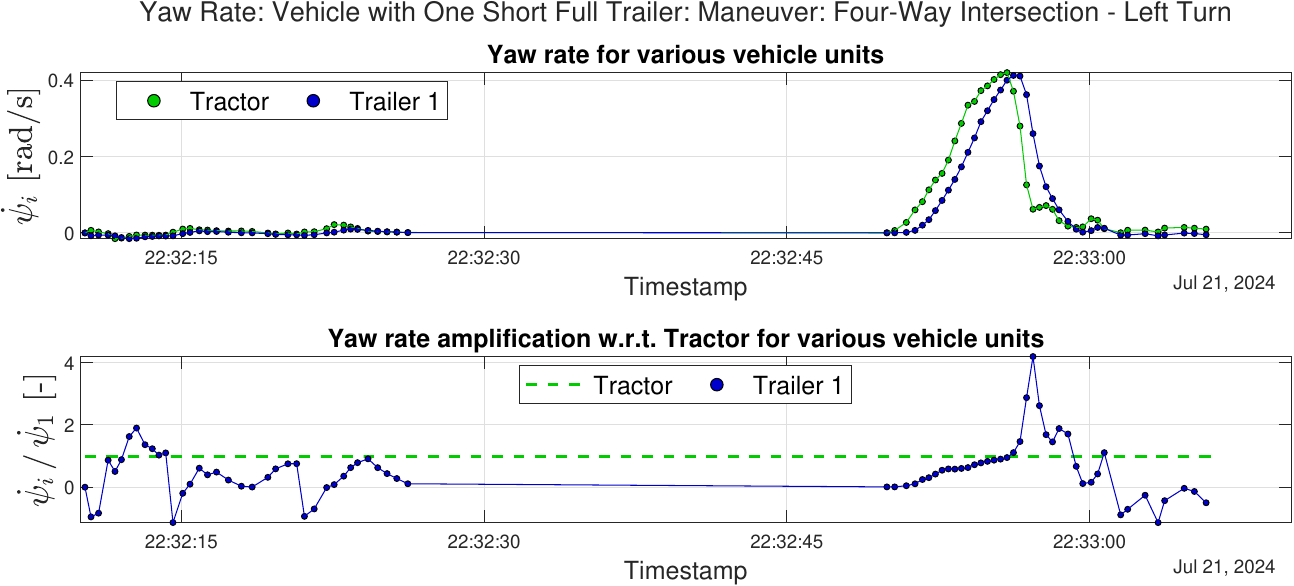}
        \caption{Configuration 3: Tractor pulling one short full trailer}
        \label{fig:yawAmp_1T_Short_sc1_config3}
    \end{subfigure}
    \hfill
    \vskip 0.01in
    \begin{subfigure}[b]{\linewidth}
        \centering
        \includegraphics[width=\linewidth]{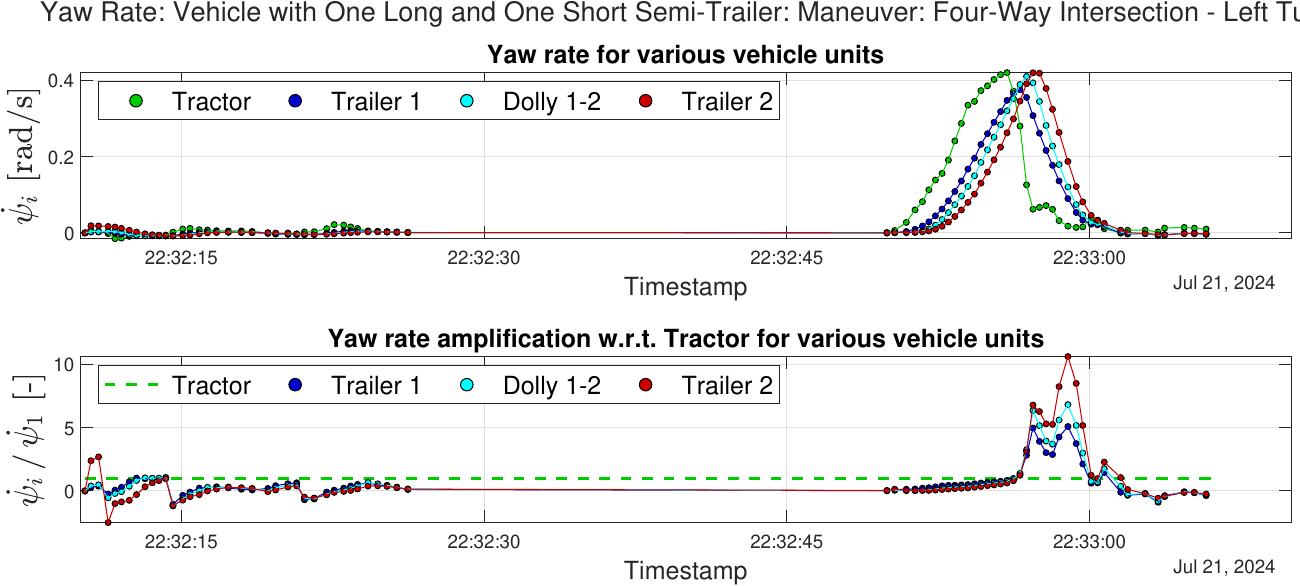}
        \caption{Configuration 4: Tractor pulling one long and one short semi-trailer}
        \label{fig:yawAmp_2T_LongShort_sc1_config4}
    \end{subfigure}
    \hfill
    \vskip 0.01in
    \begin{subfigure}[b]{\linewidth}
        \centering
        \includegraphics[width=\linewidth]{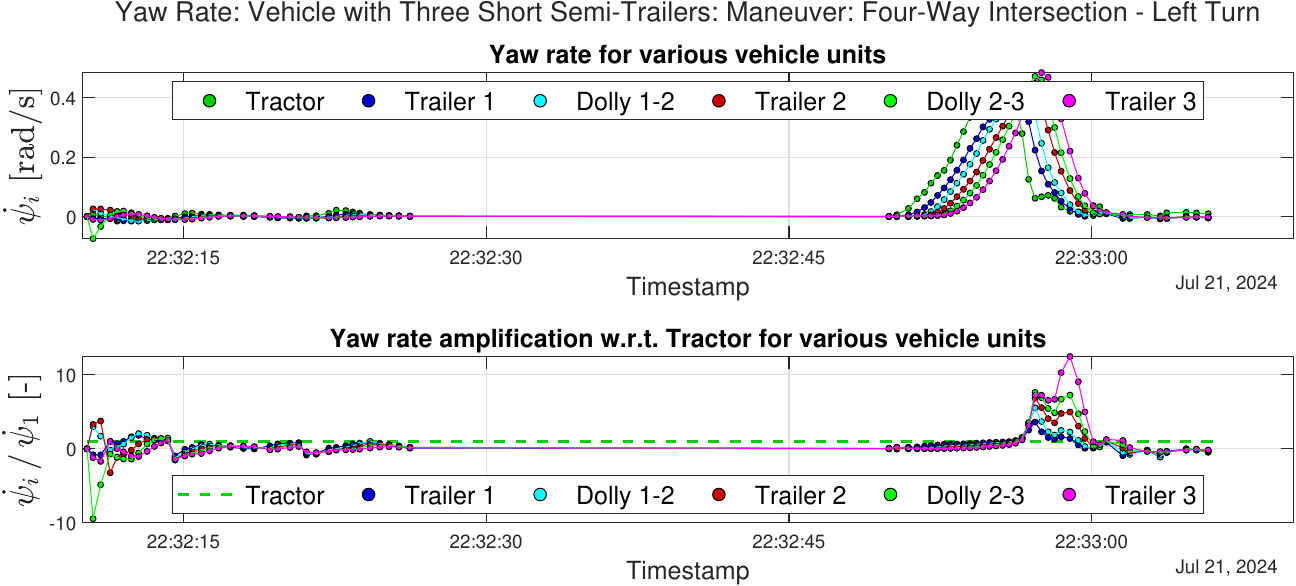}
        \caption{Configuration 5: Tractor pulling three semi-trailers}
        \label{fig:yawAmp_3T_Short_sc1_config5}
    \end{subfigure}
    \caption{Yaw rate response results for various articulated vehicle configurations having consecutive trailers connected by dollies: Four-Way Intersection}
    \label{fig:results_configs}
\end{figure}

The yaw rate response for these configurations on the left turn scenario is shown in Figure \ref{fig:results_configs}. As expected, the rearward yaw rate amplification increases for trailer units relatively behind in the articulated vehicle chain as observed in Configuration 2 (two semi-trailer). Interstingly, the maximum magnitude of rearward yaw rate amplification is comparable for Configuration 5 (three short semi-trailers) and Configuration 4 (one long, one short semi-trailer) indicating that trailers of shorter length experience overall lower rearward amplification. Moreover, on comparing Configuration 3 (short full trailer) with Configuration 1 (one semi-trailer base configuration in Figure \ref{fig:stateEvolution_sc1}), it is obeserved that while the yaw rates are comparable, the shorter full trailer has lower rearward amplification.

\subsection{Kinematic Behavior on Various Road Geometries}
\begin{figure}
    \centering
    \begin{subfigure}[b]{0.9\linewidth}
        \centering
        \includegraphics[width=0.8\linewidth]{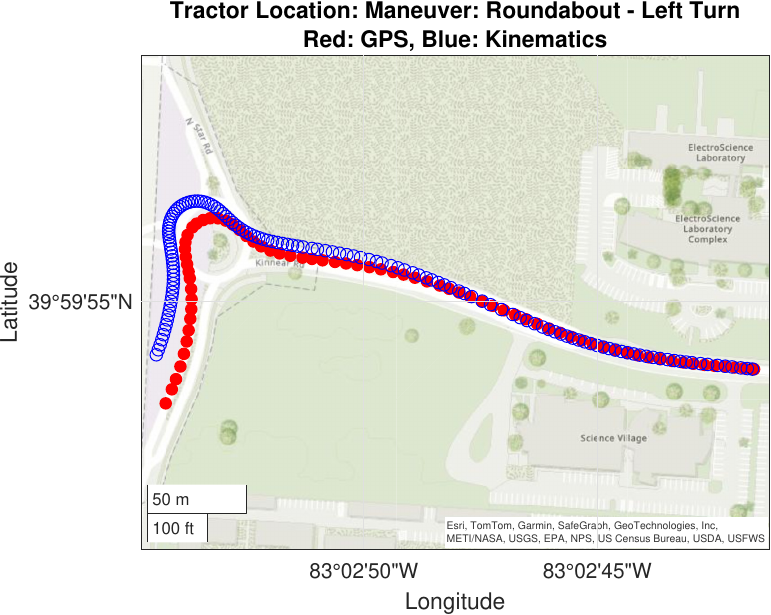}
        \caption{Scenario 1: Roundabout}
        \label{fig:tracPos_SUV_sc6}
    \end{subfigure}
    \hfill
    \vskip 0.02in
    \begin{subfigure}[b]{0.9\linewidth}
        \centering
        \includegraphics[width=0.8\linewidth]{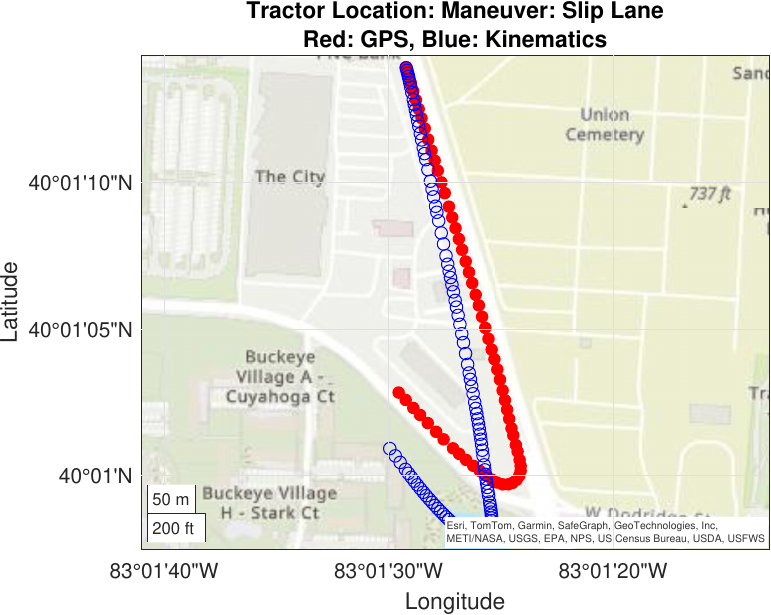}
        \caption{Scenario 2: Slip lane}
        \label{fig:tracPos_SUV_sc8}
    \end{subfigure}
    \hfill
    \vskip 0.02in
    \begin{subfigure}[b]{0.9\linewidth}
        \centering
        \includegraphics[width=0.8\linewidth]{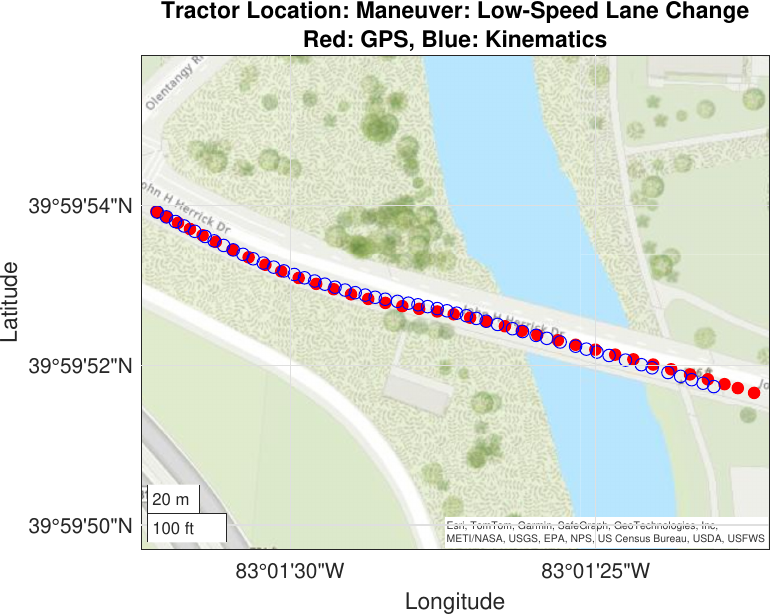}
        \caption{Scenario 3: Low-speed lane change}
        \label{fig:tracPos_SUV_sc9}
    \end{subfigure}
    \hfill
    \vskip 0.02in
    \begin{subfigure}[b]{0.9\linewidth}
        \centering
        \includegraphics[width=0.8\linewidth]{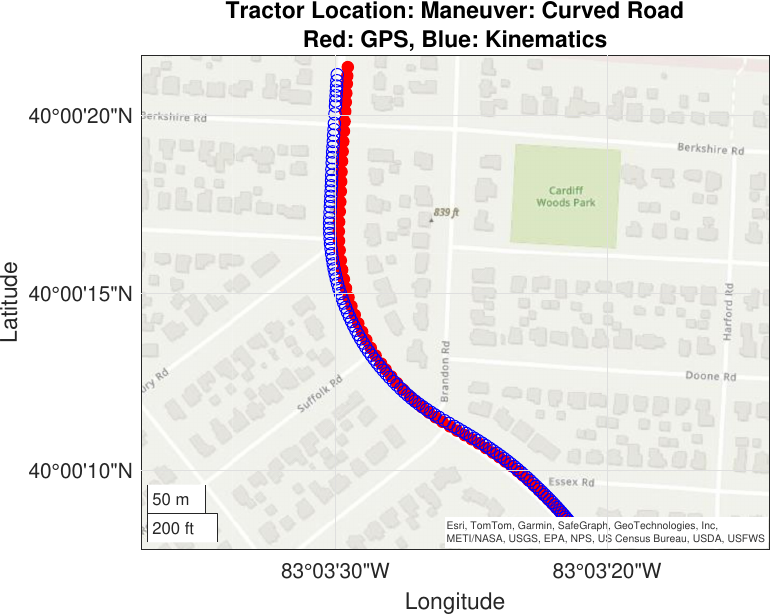}
        \caption{Scenario 4: Curved road}
        \label{fig:tracPos_SUV_sc10}
    \end{subfigure}
    \caption{Various real world scenarios tested}
    \label{fig:scenarios}
\end{figure}

The yaw rate response of kinematic model is further observed on driving data collected from four more scenarios as shown in Figure \ref{fig:scenarios}. Scenario 1 is a roundabout that is meant to test the trailer's yaw rate response on a constant curvature road. Scenario 2 is an aggressive slip lane maneuver. Scenario 3 is a low-speed lane change maneuver. Scenario 4 takes place on a curved road, and is meant to test the trailer's response to smoothly changing road geometries. 

\begin{figure}
    \centering
    \begin{subfigure}[b]{\linewidth}
        \centering
        \includegraphics[width=\linewidth]{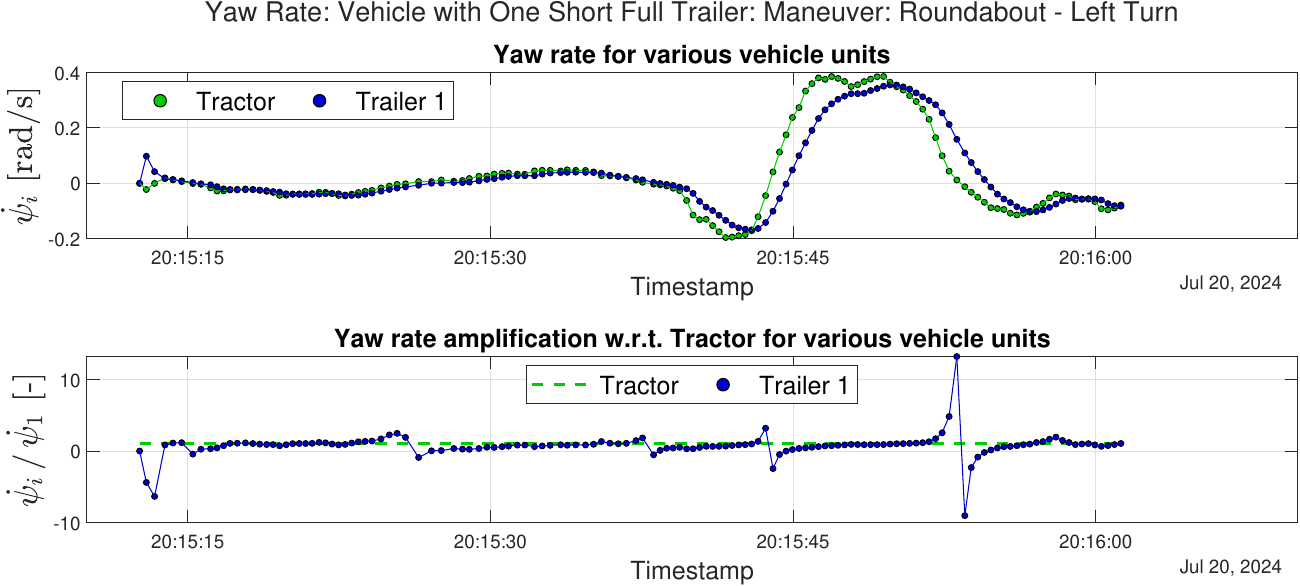}
        \caption{Scenario 1: Roundabout}
        \label{fig:yawAmp_1T_Short_sc6}
    \end{subfigure}
    \hfill
    \vskip 0.02in
    \begin{subfigure}[b]{\linewidth}
        \centering
        \includegraphics[width=\linewidth]{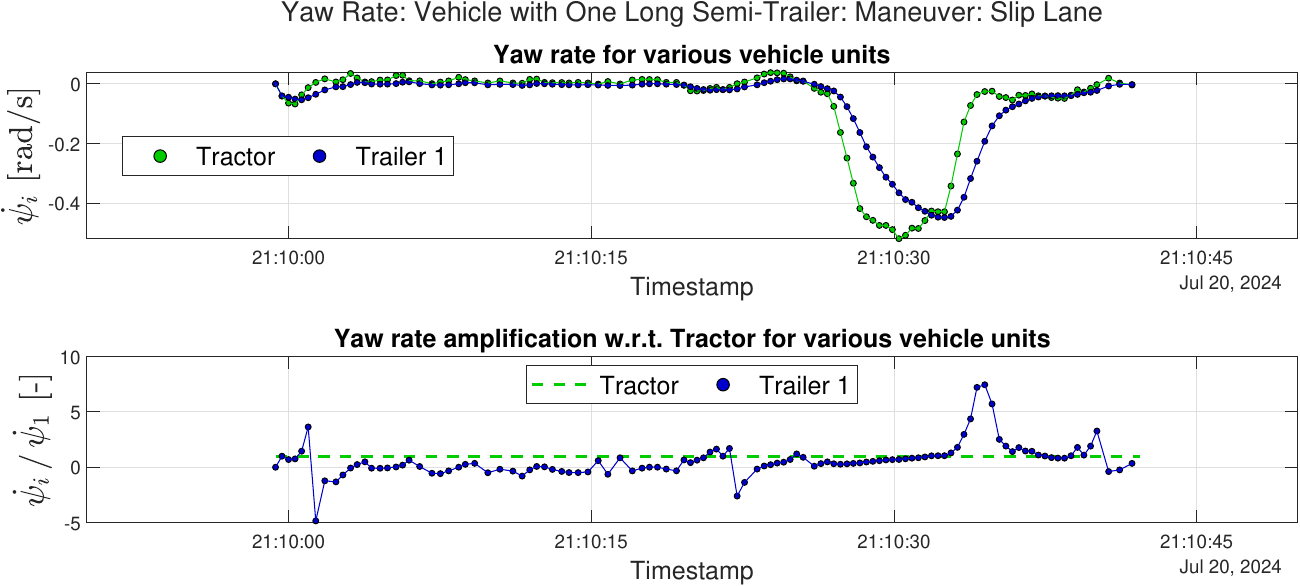}
        \caption{Scenario 2: Slip lane}
        \label{fig:yawAmp_1T_sc8}
    \end{subfigure}
    \hfill
    \vskip 0.02in
    \begin{subfigure}[b]{\linewidth}
        \centering
        \includegraphics[width=\linewidth]{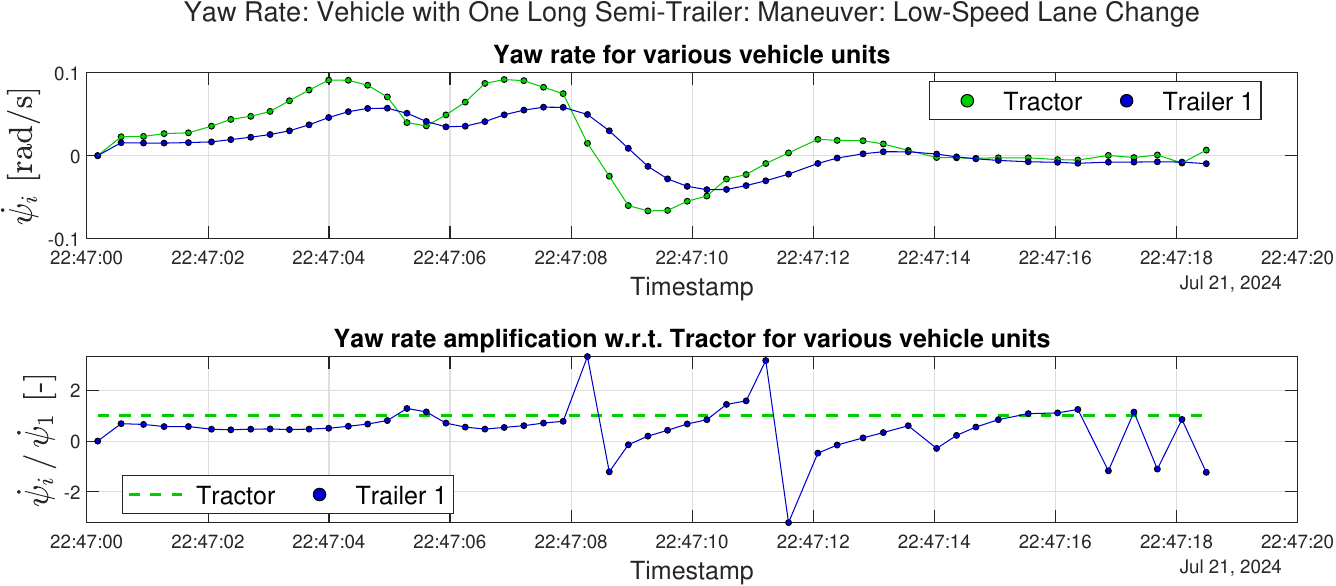}
        \caption{Scenario 3: Low-speed lane change}
        \label{fig:yawAmp_1T_sc9}
    \end{subfigure}
    \hfill
    \vskip 0.02in
    \begin{subfigure}[b]{\linewidth}
        \centering
        \includegraphics[width=\linewidth]{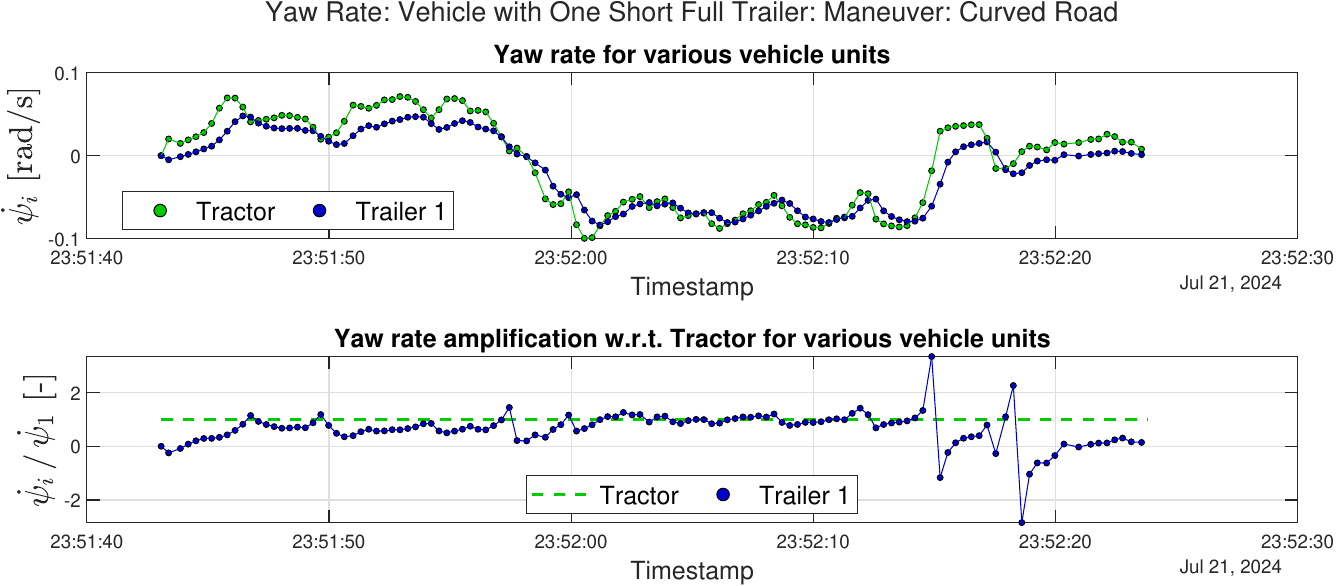}
        \caption{Scenario 4: Curved road}
        \label{fig:yawAmp_1T_Short_sc10}
    \end{subfigure}
    \caption{Yaw rate and yaw rate amplification of a one-trailer for various real world scenarios tested}
    \label{fig:scenarios_yawAmp}
\end{figure}

The yaw rate response at these scenarios is shown in Figure \ref{fig:scenarios_yawAmp}. It is observed that since the roundabout is a slow maneuver, the rearward amplification doesn't spike until the driver executes an aggressive exit out of the roundabout. Slip lane maneuver, being executed as an aggressive maneuver results in very prominent lag between tractor and trailer yaw rate, thus leading to high rearward amplification. However, it also observed that smoothly executed maneuvers such as the low-speed lane change and smooth curved road do not lead to spikes in yaw rate as well as rearward amplification.

\section{Discussion}
A symbolic algorithm to derive forward kinematic models for exotic articulated vehicle configurations was developed in this paper. These kinematic models are derived using the algorithm in polynomial computational complexity with the constraint matrix construction executing with $\mathcal{O}\left(n^2K\right)$, and the Kernel computation executing with $\mathcal{O}\left(n^2\right)$ complexity, where $n$ is the number of vehicle units and $K$ is the worst case number of wheels.

As observed in various experimental and simulation studies demonstrated in this papers, a progressive exacerbation of rearward yaw rate amplification from the first trailer unit to the rearmost trailer unit was observed. The kinematic models also demonstrate robustness against sensor noise and actuator delays. However, these models are not representative of unmodeled plant uncertainties such as hitch joint friction. Moreover, the kinematic models only consider planar rigid bodies moving in the 2-D plane, thus completely ignoring rollover behavior - an important safety consideration for articulated vehicle control. 

However, despite these limitations, the computational simplicity offered by these models make them well-suited for predictive control applications \cite{oliviera}. Moreover, despite their limitations, rearward yaw rate amplification provides a useful indication of rollover behavior for control design purposes. 

\subsection{Use Cases for these Models}
The following use cases of these models for control applications are noted:
\begin{enumerate}
    \item Offtracking control: Offtracking being a nonlinear phenomena related to the trajectory followed by each vehicle unit and wheel, generally utilizes control-affine nonlinear first-order kinematic models to develop controllers \cite{altafini_reduced_offtracking}. Recently, in the literature, kinematic models have been used to develop trajectory optimization method for offtracking control \cite{oliviera}. The modeling approach proposed in this paper aims to offer tools for extending these controllers to multi-trailer vehicles.
    \item Motion planning with directional stability and rollover considerations: These models provide a reliable indicator of rearward amplification and rollover under low-slip conditions. However, such motion planners must be accompanied by dedicated stability controllers that consider the high-slip tire behavior into account.
\end{enumerate}

\subsection{Controllability of Multi-Axle n-Trailer Vehicles}
For motion planning controller development, controllability is indeed a very relevant concern \cite{bullo_geometric_control}. It is known from prior studies \cite{general_n_trailer_properties} that the general n-trailer is small-time locally controllable. In the section on Generalized Ackermann law, we determine that the steering angles corresponding to rows eliminated from kernel computation for multi-axle units are indeed determined algebraically and uniquely, given that the independently controllable steering angles are known. Thus, this algebraic law for dependent steering actuators reduces the multi-axle system to the general n-trailer (Altafini \cite{general_n_trailer_properties}) from a strictly kinematic motion planning perspective. Hence, controllability is guaranteed, provided the kinematic control $u$ is continuous and smooth.

\section{Conclusions}
In this paper, an iterative algorithm is introduced to symbolically derive kinematic models for $n$-trailer multi-axle articulated vehicles using standard tools from Lie Group theory. The nonholonomic constraint structure for the system results in a $n$-trailer generalization of the Ackermann steering law, thus limiting the number of independently controllable wheels on each vehicle unit. The rearward yaw rate amplification behavior of the vehicle system is partially experimentally validated using real-world data from a test vehicle, where it is observed that adding additional trailers exacerbates the rearward amplification due to lag between yaw rates of consecutive vehicle units. Additionally, the robustness of the model against sensor noise, actuator delay and joint friction is studied and compared against a highly detailed dynamic model of an articulated vehicle. The symbolic algorithm presented in this paper is ideal for developing control-oriented first-order kinematic models for objectives such as trajectory optimization, offtracking control, and for motion planning with yaw rate response considerations. The kinematic models are found to be generally less realistic under high-slip conditions and their predictive capability is adversely affected due to plant uncertainties such as joint friction. However, due to their computational simplicity, these models are suitable for high-level supervisory control tasks for vehicle systems with multiple trailers.

\section*{Acknowledgments}
The authors would like to thank Dr. Ryan Chladny and Dr. Fengchen (Felix) Wang at The MathWorks Inc., and Dr. Andrea Serrani at The Ohio State University for their support and advice. 
We also thank Nur Uddin Javed and Shengzhe Tan at The Ohio State University for helping with data collection from the experimental vehicle.

\appendix

\bibliographystyle{elsarticle-num} 
\bibliography{refs_mechatronics2025}

\begin{thebibliography}{10}
\expandafter\ifx\csname url\endcsname\relax
  \def\url#1{\texttt{#1}}\fi
\expandafter\ifx\csname urlprefix\endcsname\relax\def\urlprefix{URL }\fi
\expandafter\ifx\csname href\endcsname\relax
  \def\href#1#2{#2} \def\path#1{#1}\fi

\bibitem{trigell_truck_trailer_dynamics}
A.~S. Trigell, M.~Rothhämel, J.~Pauwelussen, K.~Kural, Advanced vehicle dynamics of heavy trucks with the perspective of road safety, Vehicle System Dynamics 55~(10) (2017) 1572--1617.
\newblock \href {https://doi.org/10.1080/00423114.2017.1319964} {\path{doi:10.1080/00423114.2017.1319964}}.

\bibitem{kurtz_anderson_survey}
E.~F. Kurtz, R.~J. Anderson, Handling characteristics of car-trailer systems; a state-of-the-art survey, Vehicle System Dynamics 6~(4) (1977) 217--243.
\newblock \href {https://doi.org/10.1080/00423117708968545} {\path{doi:10.1080/00423117708968545}}.

\bibitem{vlk_lateral_dynamics}
F.~Vlk, Lateral dynamics of commercial vehicle combinations a literature survey, Vehicle System Dynamics 11~(5-6) (1982) 305--324.
\newblock \href {https://doi.org/10.1080/00423118208968702} {\path{doi:10.1080/00423118208968702}}.

\bibitem{jindra_1966}
F.~Jindra, Handling characteristics of tractor-trailer combinations, in: National Powerplant and Transportation Meetings, SAE International, 1965, p.~17.
\newblock \href {https://doi.org/https://doi.org/10.4271/650720} {\path{doi:https://doi.org/10.4271/650720}}.

\bibitem{mikulcik_1971}
E.~C. Mikulcik, \href{http://www.jstor.org/stable/44731360}{The dynamics of tractor-semitrailer vehicles: The jackknifing problem}, SAE Transactions 80 (1971) 154--168.
\newline\urlprefix\url{http://www.jstor.org/stable/44731360}

\bibitem{truck9dof}
M.~G{\"a}fvert, O.~Lindg{\"a}rde, A 9-DOF Tractor-Semitrailer Dynamic Handling Model for Advanced Chassis Control Studies, Technical Reports TFRT-7597, Department of Automatic Control, Lund Institute of Technology (LTH), 2001.

\bibitem{genta}
G.~Genta, A.~Genta, {Road Vehicle Dynamics}, World Scientific, 2017.
\newblock \href {https://doi.org/10.1142/9738} {\path{doi:10.1142/9738}}.

\bibitem{trucksim}
{Mechanical Simulation Corporation}, {TruckSim Overview}, \url{https://www.carsim.com/products/trucksim/}, accessed: 2023-08-28 (n.d.).

\bibitem{truckmaker}
{IPG Automotive GmbH}, {TruckMaker | IPG Automotive}, \url{https://ipg-automotive.com/en/products-solutions/software/truckmaker/}, accessed: 2023-08-28 (n.d.).

\bibitem{vdbs}
{The MathWorks, Inc.}, {Vehicle Dynamics Blockset - MATLAB}, \url{https://www.mathworks.com/products/vehicle-dynamics.html}, accessed: 2023-08-28 (n.d.).

\bibitem{kinematic_reduction_controllability_1}
F.~Bullo, K.~Lynch, Kinematic controllability for decoupled trajectory planning in underactuated mechanical systems, IEEE Transactions on Robotics and Automation 17~(4) (2001) 402--412.
\newblock \href {https://doi.org/10.1109/70.954753} {\path{doi:10.1109/70.954753}}.

\bibitem{laumond_controllability}
J.-P. Laumond, Controllability of a multibody mobile robot, IEEE Transactions on Robotics and Automation 9~(6) (1993) 755--763.
\newblock \href {https://doi.org/10.1109/70.265919} {\path{doi:10.1109/70.265919}}.

\bibitem{chained_form_1}
R.~Murray, S.~Sastry, Steering nonholonomic systems in chained form, in: [1991] Proceedings of the 30th IEEE Conference on Decision and Control, 1991, pp. 1121--1126 vol.2.
\newblock \href {https://doi.org/10.1109/CDC.1991.261508} {\path{doi:10.1109/CDC.1991.261508}}.

\bibitem{chained_form_sinusoids_1}
D.~Tilbury, J.-P. Laumond, R.~Murray, S.~Sastry, G.~Walsh, Steering car-like systems with trailers using sinusoids, in: Proceedings 1992 IEEE International Conference on Robotics and Automation, 1992, pp. 1993--1998 vol.3.
\newblock \href {https://doi.org/10.1109/ROBOT.1992.219988} {\path{doi:10.1109/ROBOT.1992.219988}}.

\bibitem{chained_form_sinusoids_2}
R.~Murray, S.~Sastry, Nonholonomic motion planning: steering using sinusoids, IEEE Transactions on Automatic Control 38~(5) (1993) 700--716.
\newblock \href {https://doi.org/10.1109/9.277235} {\path{doi:10.1109/9.277235}}.

\bibitem{flatness_motion_planning_1}
P.~Rouchon, M.~Fliess, J.~L\'evine, P.~Martin, Flatness and motion planning: the car with n-trailers, European Control Conference (01 1992).

\bibitem{general_n_trailer_properties}
C.~Altafini, Some properties of the general n-trailer, International Journal of Control 74~(4) (2001) 409--424.
\newblock \href {https://doi.org/10.1080/00207170010010579} {\path{doi:10.1080/00207170010010579}}.

\bibitem{flatness_motion_planning_2}
P.~Rouchon, M.~Fliess, J.~L\'evine, P.~Martin, Flatness, motion planning and trailer systems, in: Proceedings of 32nd IEEE Conference on Decision and Control, 1993, pp. 2700--2705 vol.3.
\newblock \href {https://doi.org/10.1109/CDC.1993.325686} {\path{doi:10.1109/CDC.1993.325686}}.

\bibitem{deBruin_dynamic_model}
D.~de~Bruin, P.~van~den Bosch, Modelling and control of a double articulated vehicle with four steerable axles, in: Proceedings of the 1999 American Control Conference (Cat. No. 99CH36251), Vol.~5, 1999, pp. 3250--3254 vol.5.
\newblock \href {https://doi.org/10.1109/ACC.1999.782365} {\path{doi:10.1109/ACC.1999.782365}}.

\bibitem{kim_control_all_steered}
K.-H.~Y. Young Chol~Kim, K.-D. Min, Automatic guidance control of an articulated all-wheel-steered vehicle, Vehicle System Dynamics 52~(4) (2014) 456--474.
\newblock \href {https://doi.org/10.1080/00423114.2013.831458} {\path{doi:10.1080/00423114.2013.831458}}.

\bibitem{yawrate_rwa_mpc}
M.~Keshavarz~Bahaghighat, S.~Kharrazi, M.~Lidberg, P.~Falcone, B.~Schofield, Predictive yaw and lateral control in long heavy vehicles combinations, in: 49th IEEE Conference on Decision and Control (CDC), 2010, pp. 6403--6408.
\newblock \href {https://doi.org/10.1109/CDC.2010.5717377} {\path{doi:10.1109/CDC.2010.5717377}}.

\bibitem{oliviera}
R.~Oliveira, O.~Ljungqvist, P.~F. Lima, B.~Wahlberg, Optimization-based on-road path planning for articulated vehicles⁎⁎this work was partially supported by the wallenberg ai, autonomous systems and software program (wasp) funded by the knut and alice wallenberg foundation., IFAC-PapersOnLine 53~(2) (2020) 15572--15579, 21st IFAC World Congress.
\newblock \href {https://doi.org/j.ifacol.2020.12.2402} {\path{doi:j.ifacol.2020.12.2402}}.

\bibitem{astolfi_path_tracking}
A.~Astolfi, P.~Bolzern, A.~Locatelli, Path-tracking of a tractor-trailer vehicle along rectilinear and circular paths: a lyapunov-based approach, IEEE Transactions on Robotics and Automation 20~(1) (2004) 154--160.
\newblock \href {https://doi.org/10.1109/TRA.2003.820928} {\path{doi:10.1109/TRA.2003.820928}}.

\bibitem{michalek}
M.~M. Michałek, B.~Patkowski, T.~Gawron, Modular kinematic modelling of articulated buses, IEEE Transactions on Vehicular Technology 69~(8) (2020) 8381--8394.
\newblock \href {https://doi.org/10.1109/TVT.2020.2999639} {\path{doi:10.1109/TVT.2020.2999639}}.

\bibitem{Orosco_modeling_feedback_linearization_n_trailer}
R.~Orosco-Guerrero, E.~Aranda-Bricaire, M.~Velasco-Villa, {Modeling and Dynamic Feedback Linearization of a Multi-steered N-Trailer}, IFAC Proceedings Volumes 35~(1) (2002) 103--108, 15th IFAC World Congress.
\newblock \href {https://doi.org/https://doi.org/10.3182/20020721-6-ES-1901.00267} {\path{doi:https://doi.org/10.3182/20020721-6-ES-1901.00267}}.

\bibitem{tilbury_multisteering}
D.~Tilbury, O.~Sordalen, L.~Bushnell, S.~Sastry, A multisteering trailer system: conversion into chained form using dynamic feedback, IEEE Transactions on Robotics and Automation 11~(6) (1995) 807--818.
\newblock \href {https://doi.org/10.1109/70.478428} {\path{doi:10.1109/70.478428}}.

\bibitem{ahmadian_rwa}
{Zhang, Zichen}, {Chen, Yang}, {Ahmadian, Mehdi}, Dimensionless analysis of rearward amplification of trucks with single and double trailers: A frequency analysis, SAE International Journal of Commercial Vehicles 16~(3) (2022) 231--245.
\newblock \href {https://doi.org/https://doi.org/10.4271/02-16-03-0015} {\path{doi:https://doi.org/10.4271/02-16-03-0015}}.

\bibitem{lynch_park_robotics}
K.~M. Lynch, F.~C. Park, {Modern Robotics: Mechanics, Planning, and Control}, Cambridge University Press, 2017.

\bibitem{winkler1999rollover}
C.~B. Winkler, Rollover of heavy commercial vehicles, Tech. rep., Society of Automotive Engineers, Warrendale, Pa. (1999).

\bibitem{luijten_rwa}
R.~M.~V. M.~F.J. van~de Molengraft-Luijten, I. J.M.~Besselink, H.~Nijmeijer, Analysis of the lateral dynamic behaviour of articulated commercial vehicles, Vehicle System Dynamics 50~(sup1) (2012) 169--189.
\newblock \href {https://doi.org/10.1080/00423114.2012.676650} {\path{doi:10.1080/00423114.2012.676650}}.

\bibitem{rwa_lqr}
M.~El-Gindy, B.~T. Kułakowski, X.~Tong, N.~Mrad, Rearward amplification control of a truck/full-trailer, IFAC Proceedings Volumes 31~(1) (1998) 7--15, 2nd IFAC Workshop on Advances in Automative Control 1998, Loudonville, USA, 26 February - 1 March.
\newblock \href {https://doi.org/https://doi.org/10.1016/S1474-6670(17)42170-7} {\path{doi:https://doi.org/10.1016/S1474-6670(17)42170-7}}.

\bibitem{altafini_reduced_offtracking}
C.~Altafini, Path following with reduced off-tracking for multibody wheeled vehicles, IEEE Transactions on Control Systems Technology 11~(4) (2003) 598--605.
\newblock \href {https://doi.org/10.1109/TCST.2003.813374} {\path{doi:10.1109/TCST.2003.813374}}.

\bibitem{pacejka}
H.~B. Pacejka, E.~Bakker, The magic formula tyre model, Vehicle System Dynamics 21~(sup001) (1992) 1--18.
\newblock \href {https://doi.org/10.1080/00423119208969994} {\path{doi:10.1080/00423119208969994}}.

\bibitem{armstrong_friction}
B.~Armstrong-Helouvry, Control of machines with friction, Vol. 128, Springer Science \& Business Media, 1991.

\bibitem{MathWorks_RotationalFriction}
{MathWorks}, \href{https://www.mathworks.com/help/simscape/ref/rotationalfriction.html}{Rotational friction - simscape documentation}, accessed: 2025-03-11 (2025).
\newline\urlprefix\url{https://www.mathworks.com/help/simscape/ref/rotationalfriction.html}

\bibitem{bullo_geometric_control}
F.~Bullo, A.~D. Lewis, {Geometric Control of Mechanical Systems: Modeling, Analysis, and Design for Simple Mechanical Control Systems}, Vol.~49, Springer, New York, NY, 2005.
\newblock \href {https://doi.org/10.1007/978-1-4899-7276-7} {\path{doi:10.1007/978-1-4899-7276-7}}.

\end{thebibliography}

\end{document}